\documentclass{article}
\usepackage{ijcai15}

\usepackage{times}
\usepackage{amsmath}
\usepackage{amsthm}
\usepackage{amssymb}
\usepackage{graphicx}

\usepackage{booktabs}
\usepackage{url}
\usepackage{caption}
\usepackage{subcaption}
\usepackage{booktabs}

\newcommand{\comment}[1]{}

\newtheorem{definition}{Definition}
\newtheorem{theorem}{Theorem}

\begin{document}

\title{Metareasoning for Planning Under Uncertainty}

\comment{
\author{
Christopher H. Lin \thanks{Research was performed while the author was an intern at Microsoft Research.}\\
University of Washington\\
Seattle, WA\\
chrislin@cs.washington.edu \And
Andrey Kolobov \hspace{6px} Ece Kamar \hspace{6px} Eric Horvitz\\
Microsoft Research\\
Redmond, WA\\
\{akolobov,eckamar,horvitz\}@microsoft.com }
}

\author{
 \\
 \vspace{0.5in}
{\centering
\begin{tabular}{cc}
\textbf{Christopher H. Lin} \thanks{Research was performed while the author was an intern at Microsoft Research.} & \textbf{Andrey Kolobov  \hspace{25px} Ece Kamar \hspace{25px} Eric Horvitz}\\
 University of Washington &  Microsoft Research\\
Seattle, WA & Redmond, WA\\
chrislin@cs.washington.edu & \{akolobov,eckamar,horvitz\}@microsoft.com 
\end{tabular}
}
}

\maketitle

\begin{abstract}

The conventional model for online planning under uncertainty assumes that an agent can stop and plan without incurring costs for the time spent planning.  However, planning time is not free in most real-world settings. For example, an autonomous drone is subject to nature's forces, like gravity, even while it thinks, and must either pay a price for counteracting these forces to stay in place, or grapple with the state change caused by acquiescing to them. Policy optimization in these settings requires \emph{metareasoning}---a process that trades off the cost of planning and the potential policy improvement that can be achieved. We formalize and analyze the metareasoning problem for Markov Decision Processes (MDPs). Our work subsumes previously studied special cases of metareasoning and shows that in the general case, metareasoning is at most polynomially harder than solving MDPs with any given algorithm that disregards the cost of thinking. For reasons we discuss, optimal general metareasoning turns out to be impractical, motivating approximations. We present approximate metareasoning procedures which rely on special properties of the BRTDP planning algorithm and explore the effectiveness of our methods on a variety of problems.

\end{abstract}








\section{Introduction}
\comment{
In online planning, an agent must repeatedly decide whether to invest in additional planning to improve its policy or to continue to execute a previously computed plan. We formalize and study this challenge of metareasoning about planning.  Metareasoning about base-level problem solving has been explored for guiding such challenges as probabilistic inference and decision making \cite{horvitz-ijcai89}, theorem proving \cite{horvitz-uai95,kautz-aaai02}, handling streams of problems \cite{horvitz-aij01,shahaf-ijcai09}, and search \cite{russell-aij91}. 
However, there has been no work on metareasoning for the general planning problem.  
}
\comment{We formalize metareasoning about plans and show its complexity.  Then, we explore the use of approximations of the expected value of policies that can be computed with additional planning. The greedy approximation consider bounds estimated by the RTDP algorithm. We show the behavior of the metareasoning procedure on experiments in several synthetic domains.}

Offline probabilistic planning approaches, such as policy iteration \cite{howard-60}, aim to construct a policy for every possible state before acting. In contrast, online planners, such as RTDP \cite{barto-aij95} and UCT \cite{kocsis-ecml06}, interleave planning with execution. After an agent takes an action and moves to a new state, these planners suspend execution to plan for the next step.\comment{This process leads to a more context-sensitive exploration of the policy space but requires computation for planning in addition to execution.} The more planning time they have, the better their action choices. Unfortunately, planning time in online settings is usually not free.

\comment{
introducing a metareasoning decision problem: How long should an agent plan every time it needs to make a move? This question has been sidestepped in work on online planning algorithms by assuming that taking a break from acting in order to plan is free. In real-world applications, this delay typically introduces costs.

Traditional algorithms for planning under uncertainty may not be applicable for realistic problems with large or infinite state spaces. Online algorithms such as , have been proposed to overcome this complexity and have shown to be successful in tackling much larger problems than previously possible.
Such innovation is a great boon as many real-world domains are in this class of intractable problems. However, any agent utilizing online planning must make a tradeoff between continuing to plan and executing its suboptimal plan. 
Just as a traveler cannot spend her entire day planning her vacation in her hotel room, lest she risk having no vacation at all, an agent must \emph{metareason} about the value of computation.
}

Consider an autonomous Mars rover trying to decide what to do while a sandstorm is nearing. The size and uncertainty of the domain precludes a-priori computation of a complete policy, and demands the use of online planning algorithms. Normally, the longer the rover runs its planning algorithm, the better decision it can make. However, computation costs power; moreover, if it reasons for too long without taking preventive action, it risks being damaged by the oncoming sandstorm. Or consider a space probe on final approach to a speeding comet, when the probe must plan to ensure a safe landing based on new information it gets about the comet's surface. More deliberation time means a safer landing. At the same time, if the probe deliberates for too long, the comet may zoom out of range --- a similarly undesirable outcome.


\comment{
Scenarios like these give rise to a \emph{metareasoning decision problem}: how should an agent trade off the cost of planning  and the quality of the resulting policy for the base planning task every time it needs to make a move, so as to optimize its long-term utility? By assuming the cost of computation time to be zero, the literature on online planning and reasoning, with rare exceptions \cite{horvitz-ijcai89,russell-aij91}, has side-stepped this question. }

Scenarios like these give rise to a general \emph{metareasoning decision problem}: how should an agent trade off the cost of planning
and the quality of the resulting policy for the base planning
task every time it needs to make a move, so as to optimize
its long-term utility? Metareasoning about base-level problem solving has been explored for probabilistic inference and decision making \cite{horvitz-uai87,horvitz-ijcai89}, theorem proving \cite{horvitz-uai95,kautz-aaai02}, handling streams of problems \cite{horvitz-aij01,shahaf-ijcai09}, and search \cite{russell-aij91,burns-jair13}. There has been little work exploring generalized approaches to metareasoning for planning.

We explore the general metareasoning problem for Markov decision processes (MDPs). We begin by formalizing the problem with a general but precise definition that subsumes several previously considered metareasoning models. Then, we show with a rigorous theoretical analysis that optimal general metareasoning for planning under uncertainty is at most polynomially harder than solving the original planning problem with any given MDP solver. However, this increase in computational complexity, among other reasons we discuss, renders such optimal general metareasoning impractical. The analysis raises the issue of allocating time for \emph{metareasoning itself}, and leads to an infinite regress of meta$^*$reasoning (metareasoning, metametareasoning, etc.) problems. 

We next turn to the development and testing of fast approximate metareasoning algorithms. Our procedures use the Bounded RTDP (BRTDP \cite{mcmahan-icml05}) algorithm to tackle the base MDP problem, and leverage BRTDP-computed bounds on the quality of MDP policies to reason about the value of computation. In contrast to prior work on this topic, our methods do not require any training data, precomputation, or prior information about target domains. We perform a set of experiments showing the performance of these algorithms versus baselines in several synthetic domains with different properties, and characterize their performance with a measure that we call the \emph{metareasoning gap} --- a measure of the potential for improvement from metareasoning. The experiments demonstrate that the proposed techniques excel when the metareasoning gap is large.

\section{Related Work}
\comment{
Previous work includes investigation of metareasoning for time-critical single-shot decisions \cite{horvitz-uai87,horvitz-ijcai89}, where expected value of computation is used to guide probabilistic inference. Other efforts have explored changes in the computation of single actions in search rather than plans \cite{russell-aij91,burns-jair13}.\comment{It includes calculating the value of computation in which a real-time medical assistant must find the optimal point at which to stop diagnostic inference and begin treatment of a patient.}  Beyond real-time metareasoning, several lines of work have leveraged offline learning and optimization \cite{horvitz-tr90,kautz-aaai02,russell-aij91,russell-ijcai89,horvitz-uai01}. In other work, \cite{hansen-aij01} proposed a non-myopic dynamic programming solution for single-shot problems. }

Metareasoning efforts to date have employed strategies that avoid the complexity of the general metareasoning problem for planning via relying on different kinds of simplifications and approximations. 
Such prior studies include metareasoning for time-critical decisions where expected value of computation is used to guide probabilistic inference \cite{horvitz-uai87,horvitz-ijcai89}, and work on the guiding of sequences of single actions in search \cite{russell-aij91,burns-jair13}. Several lines of work have leveraged offline learning \cite{breese-uai90,horvitz-uai01,kautz-aaai02}. Other studies have relied on optimizations and inferences that leverage the structure of problems, such as the functional relationships between metareasoning and reasoning \cite{horvitz-tr90,zilberstein-aij96}, the structure of the problem space \cite{horvitz-uai95}, and the structure of utility \cite{horvitz-aij01}.  In other work, \cite{hansen-aij01} proposed a non-myopic dynamic programming solution for single-shot problems. Finally, several planners rely on a heuristic form of online metareasoning when maximizing policy reward under computational constraints in real-world time with no ``conversion rate'' between the two \cite{kolobov-aaai12,keller-aaai15}. In contrast, our metareasoning model is unconstrained, with computational and base-MDP costs in the same ``currency.''

\comment{
Planners at the International Probabilistic Planning Competition (IPPC) \cite{IPPC-2011} also rely on a form of metareasoning, albeit only implicitly and heuristically: they try to compute a policy with high reward under computational constraints in real-world time, with no ``conversion rate'' between the two. In contrast, our metareasoning model is unconstrained, with computational and base-MDP costs in the same ``currency.'' 
}

\comment{
Many planners designed for the International Probabilistic Planning Competition (e.g., \cite{kolobov-icaps12,kolobov-aaai12} optimize base-MDP costs under computational constraints specified in 
real-world-time, with no “conversion rate” between the two. In contrast, our model is unconstrained, with computational and base-MDP costs in the same “currency.” 
}
Our investigation also has connections to research on allocating time in a system composed of multiple sensing and planning components \cite{zilberstein-aij96,zilberstein-ijcai93}, on optimizing portfolios of planning strategies in scheduling applications \cite{dean-aij95}, and on choosing actions to explore in Monte Carlo planning \cite{hay-uai12}. In other related work, \cite{chanel-icaps14} consider how best to plan on one thread, while a separate thread processes execution. 

\section{Preliminaries}
A key contribution of our work is formalizing the metareasoning problem for planning under uncertainty. We build on the framework of stochastic shortest path (SSP) MDPs with a known start state. This general MDP class includes finite-horizon and discounted-reward MDPs as special cases \cite{bertsekas-book}, and can also be used to approximate partially observable MDPs with a fixed initial belief state. 
An SSP MDP $M$ is a tuple $\langle S, A, T, C,s_0, s_g \rangle$, where $S$ is a finite set of states, $A$ is a set of actions 
that the agent can take,
$T: (S, A, S) \rightarrow [0,1]$ is a transition function, $C: (S, A) \rightarrow \mathbb{R}$ is a cost function, $s_0 \in S$ is the start state, and $s_g$ is the goal state. An SSP MDP must have a complete proper policy, a policy that leads to the goal from any state with probability 1, and all improper policies must accumulate infinite cost from every state from which they fail to reach the goal with a positive probability. The objective is to find a Markovian policy $\pi_{s_0}:S \rightarrow A$ with the minimum expected cost of reaching the goal from the start state $s_0$ --- in SSP MDPs, at least one policy of this form is globally optimal.

Without loss of generality, we assume an SSP MDP to have a specially designated \texttt{NOP} (``no-operation'') action. \texttt{NOP} is an action the agent chooses when it wants to ``idle'' and ``think/plan'', and its semantic meaning is problem-dependent. For example, in some MDPs, choosing \texttt{NOP} means staying in the current state for one time step, while in others it may mean allowing a tidal wave to carry the agent to another state. Designating an action as \texttt{NOP} does not change SSP MDPs' mathematical properties, but plays a crucial role in our metareasoning formalization.




\section{Formalization and Theoretical Analysis of Metareasoning for MDPs}\label{sec:theory}

The online planning problem of an agent, which involves choosing an action to execute in any given state, is represented as an SSP MDP that encapsulates the dynamics of the environment and costs of acting and thinking. We call this problem the \emph{base problem}. The agent starts off in this environment with some default policy, which can be as simple as random or guided by an unsophisticated heuristic. The agent's metareasoning problem, then, amounts to deciding, at every step during its interaction with the environment, between improving its existing policy or using this policy's recommended action while paying a cost for executing \emph{either of these options}, so as to minimize its expected cost of getting to the goal. 

\comment{
Existing theory of sequential decision making aims at optimizing the utility of the course of action an agent should pursue in the scenario at hand, and assumes that the agent does not incur a cost for \emph{computing} this course of action. The metareasoning problem lifts this assumption by allowing an agent to start off with an arbitrary policy and choosing at every step between spending several computation cycles on improving it or executing an action recommended by the existing policy, \emph{while paying a cost for either of these options}. The question is, then: how should an agent decide between the two so as to maximize its long-term utility?
}

Besides the agent's state in the base MDP, which we call the \emph{world state}, the agent's metareasoning decisions are conditioned on the algorithm the agent uses for solving the base problem, i.e., intuitively, on the agent's thinking process. To abstract away the specifics of this planning algorithm for the purposes of metareasoning formalization, we view it as a black-box MDP solver and represent it, following the Church-Turing thesis, with a Turing machine $B$ that takes a base SSP MDP $M$ as input. 
\comment{
Besides the dynamics of the environment and relative costs of acting and thinking, all of which can be captured by an SSP MDP that we call the \emph{base problem}, the answer to this question depends on the properties of an agent's process of thinking about the base problem. To abstract away the details of this planning procedure, we view the algorithm used by the agent as a black-box MDP solver and represent it, following the Church-Turing thesis, with a Turing machine $B$ that takes a base SSP MDP $M$ as input.  
}
In our analysis, we assume the following about Turing machine $B$'s operation:

\begin{itemize}
\item{} $B$ is deterministic and halts on every valid base MDP $M$. This assumption does not affect the expressiveness of our model, since randomized Turing machines can be trivially simulated on deterministic ones, e.g., via seed enumeration (although potentially at an exponential increase in time complexity). At the same time, it greatly simplifies our theorems.

\item{} An agent's thinking cycle corresponds to $B$ executing a single instruction.

\item{} A \emph{configuration} of $B$ is a combination of $B$'s tape contents, state register contents, head position, and next input symbol. It represents the state of the online planner in solving the base problem $M$. We denote the set of all configurations $B$ ever enters on a given input MDP $M$ as $X^{B(M)}$. We assume that $B$ can be paused after executing $y$ instructions, and that its configuration at that point can be mapped to an action for any world state $s$ of $M$ using a special function $f: S \times X^{B(M)} \rightarrow A$ in time polynomial in $M$'s flat representation. The number of instructions needed to compute $f$ is not counted into $y$. That is, an agent can stop thinking at any point and obtain a policy for its current world state.


\item{}An agent is allowed to ``think'' (i.e., execute $B$'s instructions) only by choosing the \texttt{NOP} action. If an agent decides to resume thinking after pausing $B$ and executing a few actions, $B$ re-starts from the configuration in which it was last paused. 
\end{itemize}

\noindent
We can now define metareasoning precisely:

\begin{definition}\label{def:meta}Metareasoning Problem. Consider an MDP $M = \langle S, A, T, C,s_0, s_g \rangle$ and an SSP MDP solver represented by a deterministic Turing machine $B$. Let $X^{B(M)}$ be the set of all configurations $B$ enters on input $M$, and let $T^{B(M)}:X^{B(M)} \times X^{B(M)} \rightarrow \{0,1\}$ be the (deterministic) transition function of $B$ on $X^{B(M)}$. A metareasoning problem for $M$ with respect to $B$, denoted $\mathtt{Meta}_{B}(M)$ is an MDP $\langle S^m, A^m, T^m, C^m, s^m_0, s^m_g \rangle$ s.t.

\begin{itemize}

\item{} $S^m = S \times X^{B(M)}$
\item{} $A^m = A$
\item{} $T^m((s, \chi), a, (s', \chi'))$
\begin{align*}
    =  \begin{cases}
    T(s, a, s') \quad \text{if $a \neq $ \texttt{NOP}, $\chi = \chi'$, and $a =f(s, \chi)$} \\
    T(s, a, s')\cdot T^{B(M)}(\chi, \chi') \quad \text{if $a = $ \texttt{NOP}} \\
    \text{0 otherwise}
    \end{cases}
    \end{align*}

\item{}$C^m((s, \chi), a, (s', \chi')) = C(s, a, s')$ if $T(s, a, s') \neq 0$, and 0 otherwise
\item{}$s^m_0 = (s_0, \chi_0)$, where $\chi_0$ is the first configuration $B$ enters on input $M$
\item{}$s^m_g = (s_g, \chi)$, where $\chi$ is any configuration in $X^{B(M)}$
\end{itemize}

\noindent
Solving the metareasoning problem means finding a policy for $\mathtt{Meta}_{B}(M)$ with the lowest expected cost of reaching $s^m_g$.\\

\end{definition}

This definition casts a metareasoning problem for a base MDP as another MDP (a \emph{meta-MDP}). 
Note that in $\texttt{Meta}_{B}(M)$, an agent \emph{must} choose either \texttt{NOP} or an action currently recommended by $B(M)$; in other cases, the transition probability is 0. Thus, $\texttt{Meta}_{B}(M)$'s definition essentially forces an agent to switch between two ``meta-actions'': thinking or acting in accordance with the current policy. 


Modeling an agent's reasoning process with a Turing machine allows us to see that at every time step the metareasoning decision depends on the combination of the current world state and the agent's ``state of mind,'' as captured by the Turing machine's current configuration. In principle, this decision could depend on the entire history of the two, but the following theorem implies that, as for $M$, at least one optimal policy for $\texttt{Meta}_{B}(M)$ is always Markovian. 


\comment{
\begin{theorem}\label{thm:inher} If the base MDP $M$ is an SSP MDP, then $\texttt{Meta}_{B}(M)$ is an SSP MDP as well, provided that $B$ on $M$ eventually enters an absorbing region consisting of configurations that represent proper policies for $M$ and remains there until it halts. If the base MDP $M$ is an infinite-horizon discounted-reward MDP, then so is $\texttt{Meta}_{B}(M)$. If the base MDP $M$ is a finite-horizon MDP, then so is $\texttt{Meta}_{B}(M)$. 
\end{theorem}
}

\begin{theorem}\label{thm:inher} If the base MDP $M$ is an SSP MDP, then $\texttt{Meta}_{B}(M)$ is an SSP MDP as well, provided that $B$ halts on $M$ with a proper policy. If the base MDP $M$ is an infinite-horizon discounted-reward MDP, then so is $\texttt{Meta}_{B}(M)$. If the base MDP $M$ is a finite-horizon MDP, then so is $\texttt{Meta}_{B}(M)$. 
\end{theorem}

\begin{proof} 
Verifying the result for finite-horizon and infinite-horizon discounted-reward MDPs $M$ is trivial, since the only requirement $\texttt{Meta}_{B}(M)$ must satisfy in these cases is to have a finite horizon or a discount factor, respectively.

If $M$ is an SSP MDP, then, per the SSP MDP definition \cite{bertsekas-book}, to ascertain the theorem's claim we need to verify that (1) $\texttt{Meta}_{B}(M)$ has at least one proper policy and (2) every improper policy in $\texttt{Meta}_{B}(M)$ accumulates an infinite cost from some state.

To see why (1) is true, recall that $\texttt{Meta}_{B}(M)$'s state space is formed by all configurations Turing machine $B$ enters on $M$. 
Consider any state  $(s'_0,\chi'_0)$ of $\texttt{Meta}_{B}(M)$. Since $B$ is deterministic, as stated in Section 3, the configuration $\chi'_0$ lies in the linear sequence of configurations between the ``designated'' initial configuration $\chi_0$ and the final proper-policy configuration that B enters according to the theorem. Thus, $B$ can reach a proper-policy configuration from $\chi'_0$. Therefore, let the agent starting in the state $(s'_0,\chi'_0)$ of $\texttt{Meta}_{B}(M)$ choose \texttt{NOP} until $B$ halts, and then follow the proper policy corresponding to $B$'s final configuration until it reaches a goal state $s_g$ of $M$. This state corresponds to a goal state $(s_g, \chi)$ of $\texttt{Meta}_{B}(M)$. Since this construct works for any $(s'_0,\chi'_0)$, it gives a complete proper policy for $\texttt{Meta}_{B}(M)$.

\comment{
By the theorem's statement, at least one such configuration $\chi$ must exist, so the above gives a proper policy for $\texttt{Meta}_{B}(M)$, one that reaches the goal almost surely from $(s_0,\chi_0)$. 
}

To verify (2), consider any policy $\pi^m$ for $\texttt{Meta}_{B}(M)$ that with a positive probability fails to reach the goal. Any infinite trajectory of $\pi^m$ that fails to reach the goal can be mapped onto a trajectory in $M$ that repeats the action choices of $\pi^m$'s trajectory in $M$'s state space $S$. Since $M$ is an SSP MDP, this projected trajectory must accumulate an infinite cost, and therefore the original trajectory in $\texttt{Meta}_{B}(M)$ must do so as well, implying the desired result.\\
\end{proof}

We now present two results to address the difficulty of metareasoning. 

\begin{theorem}
\label{thm:poly} For an SSP MDP $M$ and a deterministic Turing machine $B$ representing a solver for $M$, the time complexity of $\texttt{Meta}_{B}(M)$ is at most polynomial in the time complexity of executing $B$ on $M$.

\comment{\label{thm:poly} For an SSP MDP $M$ and a deterministic Turing machine $B$ representing a solver for $M$, the time complexity of $\texttt{Meta}_{B}(M)$ is at most polynomial in the time complexity of $B$, by running $B$ on $M$ to construct $\texttt{Meta}_{B}(M)$.}
\end{theorem}
\begin{proof}
The main idea is to construct the MDP representing $\texttt{Meta}_{B}(M)$ by simulating $B$ on $M$. Namely, we can run $B$ on $M$ until it halts and record every configuration $B$ enters to obtain the set $X$. Given $X$, we can construct $S^m = S \times X$ and all other components of $\texttt{Meta}_{B}(M)$ in time polynomial in $|X|$ and $|M|$. Constructing $X$ itself takes time proportional to running time of $B$ on $M$. Since, by Theorem \ref{thm:inher}, $\texttt{Meta}_{B}(M)$ is an SSP MDP and hence can be solved in time polynomial in the size of its components, e.g., by linear programming, the result follows. \\
\end{proof}

\begin{theorem}\label{thm:pcomp}
Metareasoning for SSP MDPs is $P$-complete under $NC$-reduction. (Please see the appendix for proof.)
\end{theorem}

At first glance, the results above look encouraging. 
However, upon closer inspection they reveal several subtleties making optimal metareasoning utterly impractical. First, although both SSP MDPs and their metareasoning counterparts with respect to an optimal polynomial-time solver are in $P$, doing metareasoning for a given MDP $M$ is appreciably more expensive than solving \emph{that MDP itself}. \comment{Indeed, at least if the metareasoning problem is derived from $M$ literally as Definition \ref{def:meta} dictates, the time for solving it will usually be characterized by a polynomial of a significantly higher degree than required for $M$. 
This additional complexity necessitates planning for the process of metareasoning, i.e., doing meta-metareasoning, and in general causes the issue of infinite regress--an unbounded nested sequence of ever-costlier reasoning problems.}Given that the additional complexity due to metareasoning cannot be ignored, the agent now faces the new challenge of allocating computational time between metareasoning and planning for the base problem. This challenge is a meta-metareasoning problem, and ultimately causes infinite regress, an unbounded nested sequence of ever-costlier reasoning problems.

Second, constructing $\texttt{Meta}_{B}(M)$ by running $B$ on $M$, as the proof of Theorem \ref{thm:poly} proceeds, may entail solving $M$ in the process of metareasoning. While the proof doesn't show that this is the only way of constructing $\texttt{Meta}_{B}(M)$, without making additional assumptions about $B$'s operation one cannot exclude the possibility of having to run $B$ until convergence and thereby completely solving $M$ even before $\texttt{Meta}_{B}(M)$ is fully formulated. Such a construction would defeat the purpose of metareasoning.


Third, the validity of Theorems \ref{thm:poly} and \ref{thm:pcomp} relies on an implicit crucial assumption that the transitions of solver $B$ on the base MDP $M$ are known in advance. Without this knowledge, $\texttt{Meta}_{B}(M)$ turns into a reinforcement learning problem \cite{sutton-98}, which further increases the complexity of metareasoning and the need for simulating $B$ on $M$. Neither of these is viable in reality.

The difficulties with optimal metareasoning motivate the development of approximation procedures. In this regard, the preceding analysis provides two important insights. It suggests that, since running $B$ on $M$ until halting is infeasible, it may be worth trying to \emph{predict} $B$'s progress on $M$. 
Many existing MDP algorithms have clear operational patterns, e.g., evaluating policies in the decreasing order of their cost, as policy iteration does \cite{howard-60}. Regularities like these can be of value in forecasting the benefit of running $B$ on $M$ for additional cycles of thinking. We now focus on exploring approximation schemes that can leverage these patterns.

\section{Algorithms for Approximate Metareasoning}
Our approach to metareasoning is guided by \emph{value of computation} (VOC) analysis. In contrast to previous work that formulates $VOC$ for single actions or decision-making problems \cite{horvitz-uai87,horvitz-ijcai89,russell-aij91}, we aim to formulate $VOC$ for online planning. For a given metareasoning problem $\texttt{Meta}_{B}(M)$, $VOC$ at any encountered state $s^m = (s, \chi)$ is exactly the difference between the Q-value of the agent following $f(s, \chi)$ (the action recommended by the current policy of the base MDP $M$) and the Q-value of the agent taking \texttt{NOP} and thinking:
\begin{align}
VOC(s^m) = Q^*(s^m, f(s,\chi)) - Q^*(s^m, \texttt{NOP}). 
\end{align}
$VOC$ captures the difference in long-term utility between thinking and acting as determined by these Q-values. An agent should take the \texttt{NOP} action and think when the $VOC$ is positive. Our technique aims to evaluate $VOC$ by estimating $Q^*(s^m, f(s,\chi))$ and $Q^*(s^m, \texttt{NOP})$. 
However, attempting to estimate these terms in a near-optimal manner ultimately runs into the same difficulties as solving $\texttt{Meta}_{B}(M)$, such as simulating the agent's thinking process many steps into the future, and is likely infeasible. Therefore, fast approximations for the Q-values will generally have to rely on simplifying assumptions. 
\comment{However, we observe that the optimal policy for $\texttt{Meta}_{B}(M)$ is fully dictated by VOC's sign:
\begin{equation}\label{eq:obs}
	\pi^*(s^m) =  \begin{cases}
	\texttt{NOP} \quad \text{if $\mbox{VOC}(s^m) \geq 0$} \\
  f(s, \chi) \quad \text{if $\mbox{VOC}(s^m) < 0$}
    \end{cases}
    \end{equation}
This equation tells us that the agent's decision making can be robust to losses of accuracy that come from approximations. }We rely on performing greedy metareasoning analysis as has been done in past studies of metareasoning \cite{horvitz-ijcai89,russell-aij91}:

\vspace{0.05in}

\noindent
\emph{\textbf{Meta-Myopic Assumption.} In any state $s^m$ of the meta-MDP, we assume that after the current step, the agent will never again choose \texttt{NOP}, and hence will never change its policy.}

\vspace{0.05in}

This meta-myopic assumption is important in allowing us to reduce $VOC$ estimation to predicting the improvement in the value of the base MDP policy following a single thinking step. 
The weakness of this assumption is that opportunities for subsequent policy improvements are overlooked. In other words, the $VOC$ computation only reasons about the current thinking opportunity. Nonetheless, in practice, we compute $VOC$ at every timestep, so the agent \emph{can} still think later. Our experiments show that our algorithms perform well in spite of their meta-myopicity.

\comment{In practice, the metareasoning agent won't be able to compute the Q-value functions and will instead hold some belief about their values and how they will change with further computation, so $Q$ is a random variable over the agent's beliefs. }

\subsection{Implementing Metareasoning with BRTDP}

We begin the presentation of our approximation scheme with the selection of $B$, the agent's thinking algorithm. 
Since approximating $Q^*(s^m, f(s,\chi))$ and $Q^*(s^m, \texttt{NOP})$ essentially amounts to assessing policy values, we would like an online planning algorithm that provides efficient policy value approximations, preferably with some guarantees. Having access to these policy value approximations enables us to design approximate metareasoning algorithms that can evaluate $VOC$ efficiently in a domain-independent fashion.

One algorithm with this property is Bounded RTDP (BRTDP) \cite{mcmahan-icml05}. It is an anytime planning algorithm based on RTDP \cite{barto-aij95}. Like RTDP, BRTDP maintains a lower bound on an MDP's optimal value function $V^*$, which is repeatedly updated via Bellman backups as BRTDP simulates trials/rollouts to the goal, making BRTDP's configuration-to-configuration transition function $T^{B(M)}(\chi, \chi')$ stochastic. A key difference is that in addition to maintaining a lower bound, it also maintains an upper bound, updated in the same conceptual way as the lower one. If BRTDP is initialized with a \emph{monotone} upper-bound heuristic, then the upper-bound decreases monotonically as BRTDP runs. The construction of domain-independent monotone bounds is beyond the scope of this paper, but is easy for the domains we study in our experiments. Another key difference between BRTDP and RTDP is that if BRTDP is stopped before convergence, it returns an action greedy with respect to the upper, not lower bound. This behavior guarantees that the expected cost of a policy returned at any time by a monotonically-initialized BRTDP is no worse than BRTDP's current upper bound. Our metareasoning algorithms utilize these properties to estimate $VOC$. In the rest of the discussion, we assume that BRTDP is initialized with a monotone upper-bound heuristic.

\subsection{Approximating VOC}

We now show how BRTDP's properties help us with estimating the two terms in the definition of $VOC$, $Q^*(s^m, f(s,\chi))$ and $Q^*(s^m, \texttt{NOP})$. We first assume that one ``thinking cycle'' of BRTDP (i.e., executing \texttt{NOP} once and running BRTDP in the meantime, resulting in a transition from BRTDP's current configuration $\chi$ to another configuration $\chi'$) corresponds to completing some fixed number of BRTDP trials from the agent's current world state $s$. 

\subsubsection{Estimating $\mathbf{Q^*(s^m, \texttt{NOP})}$}

We first describe how to estimate the value of taking the \texttt{NOP} action (thinking). At the highest level, this estimation first involves writing down an expression for $Q^*(s^m, \texttt{NOP})$, making a series of approximations for different terms, and then modeling the behavior of how BRTDP's upper bounds on the Q-value function drop in order to compute the needed quantities. 

When opting to think by choosing \texttt{NOP}, the agent may transition to a different world state while \emph{simultaneously} updating its policy for the base problem. Therefore, we can express $Q^*(s^m, \texttt{NOP}) = $
\begin{align}\label{eq:nopor}
\sum_{s'} T(s, \texttt{NOP}, s') \sum_{\chi'} T^{B(M)}(\chi, \chi')V^*((s',\chi')).
\end{align}
Because of meta-myopicity, we have $V^*((s', \chi')) = V^{\chi'}(s')$ where $V^{\chi \prime}$ is the value function of the policy corresponding to $\chi'$ in the base MDP. However, this expression cannot be efficiently evaluated in practice, since we do not know BRTDP's transition distribution $T^{B(M)}(\chi, \chi')$ nor the state values $V^{\chi'}(s')$, forcing us to make further approximations. To do so, we assume $V^{\chi'}$ and $Q^{\chi'}$ are random variables, and rewrite $\sum_{\chi'}T^{B(M)}(\chi, \chi')V^{\chi'}(s')$ = 
\begin{align}
\sum_{a} P(A^{\chi'}_{s'} = a) E[Q^{\chi'}(s', a) | A^{\chi'}_{s'} = a].
\end{align}
where the random variable $A^{\chi'}_{s'}$ takes value $a$ iff $f(s', \chi') = a$ after one thinking cycle in state $(s, \chi)$. Intuitively, $P(A^{\chi'}_{s'} = a)$ denotes the probability that BRTDP will recommend action $a$ in state $s'$ after one thinking cycle. 
Now, let us denote the Q-value upper bound corresponding to BRTDP's current configuration $\chi$ as $\overline{Q}^\chi$. This value is \emph{known}. Then, let the upper bound corresponding to BRTDP's next configuration $\chi'$, be  $\overline{Q}^{\chi'}$. Because we do not know $\chi'$, this value is \emph{unknown}, and is a random variable. Because BRTDP selects actions greedily w.r.t. the upper bound, we follow this behavior and use the upper bound to estimate Q-value by assuming that $Q^{\chi'} = \overline{Q}^{\chi'}$. Since the value of $\overline{Q}^{\chi'}$ is unknown at the time of the $VOC$ computation, $P(A^{\chi'}_{s'} = a)$ and $E[\overline{Q}^{\chi'}(s', a) | A^{\chi'}_{s'} = a]$ are computed by integrating over the possible values of $\overline{Q}^{\chi'}$. We have that $E[\overline{Q}^{\chi'}(s', a) | A^{\chi'}_s = a] = $ 

\begin{align*}
\int_{\overline{Q}^{\chi'}(s', a)}\overline{Q}^{\chi'}(s',a) \frac{P(A^{\chi'}_{s'} = a | \overline{Q}^{\chi'}(s', a)) P(\overline{Q}^{\chi'}(s', a))}{P(A^{\chi'}_{s'} = a)},
\end{align*}
and 
$P(A^{\chi'}_{s'} = a)=$
\begin{align*}
\int_{\overline{Q}^{\chi'}(s', a)} P(\overline{Q}^{\chi'}(s',a))  \prod_{a_i \neq a} P(\overline{Q}^{\chi'}(s', a_i) > P(\overline{Q}^{\chi'}(s', a)).
\end{align*}
Therefore, we must model the distribution that $\overline{Q}^{\chi'}$ is drawn from. We do so  by modeling the change $\Delta \overline{Q} = \overline{Q}^{\chi} - \overline{Q}^{\chi'}$, due to a single BRTDP thinking cycle that corresponds to a transition from configuration $\chi$ to $\chi'$. Since $\overline{Q}^{\chi}$ is known and fixed, estimating a distribution over possible $\Delta \overline{Q}$ gives us a distribution over $\overline{Q}^{\chi'}$.


Let $\hat{\Delta} \overline{Q}_{s,a}$ be the change in $\overline{Q}_{s,a}$ resulting from the most recent thinking cycle for some state $s$ and action $a$. We first assume that the change resulting from an additional cycle of thinking, $\Delta \overline{Q}_{s,a}$, will be no larger than the last change, $\Delta \overline{Q}_{s,a} \leq \hat{\Delta} \overline{Q}_{s,a}$. This assumption is reasonable, because we can expect the change in bounds to decrease as BRTDP converges to the true value function. Given this assumption, we must choose a distribution $D$ over the interval $[0, \hat{\Delta}\overline{Q}_{s,a}]$ such that for the next thinking cycle, $\Delta \overline{Q}_{s,a} \sim D$. Figure \ref{drops}a illustrates these modeling assumptions for two hypothetical actions, $a_1$ and $a_2$. 

One option is to make $D$ uniform, so as to represent our poor knowledge about the next bound change. Then, computing $P(A^{\chi'}_{s'} = a)$ involves evaluating an integral of a polynomial of degree $O(|A|)$ (the product of $|A|-1$ CDF's and $1$ PDF), and computing $E[\overline{Q}^{\chi'}(s', a) | A^{\chi'}_s = a]$ also entails evaluating an integral of degree $O(|A|)$, and thus computing these quantities for all actions in a state can be computed in time $O(|A|^2)$. Since the overall goal of this subsection, approximating $Q^*(s^m, \texttt{NOP})$, requires computing $P(A^{\chi'}_{s'} = a)$ for all actions in all states where \texttt{NOP} may lead, assuming there are no more than $K << |A|$ such states, the complexity becomes $O(K|A|^2)$ for each state visited by the agent on its way to the goal.

A weakness of this approach is that the changes in the upper bounds for different actions are modeled independently. 
For example, if the upper bounds for two actions in a given state decrease by a large amount in the previous thinking step, then it is unlikely that in the next thinking step one of them will drop dramatically while the other drops very little. This independence can cause the amount of uncertainty in the upper bound at the next thinking step to be overestimated, leading to $VOC$ being overestimated as well. 

Therefore, we create another version of the algorithm assuming that the speed of decrease in Q-value upper bounds for all actions are perfectly correlated; all ratios between future drops in the next thinking cycle are equal to the ratios between the observed drops in the last thinking cycle. Formally, for a given state $s$, we let $\rho \sim$ Uniform[0, 1]. Then, let $\Delta \overline{Q}_{s,a} = \rho \cdot \hat{\Delta}\overline{Q}_{s,a}$ for all actions $a$.

Now, to compute $P(A^{\chi'}_{s'} = a)$, for each action $a$, we represent the range of its possible future Q-values $\overline{Q}^{\chi'}_{s,a}$ with a line segment $l_a$ on the unit interval [0,1] where $l_a(0) = \overline{Q}^{\chi}_{s,a}$ and $l_a(1) = \overline{Q}^{\chi}_{s,a} - \Delta \overline{Q}_{s,a}$. Then, $P(A^{\chi'}_{s'} = a)$ is simply the proportion of $l_a$ which lies below all the other lines representing all other actions. We can na{\"i}vely compute these probabilities in time $O(|A|^2)$ by enumerating all intersections. Similarly, computing $E[\overline{Q}^{\chi'}(s', a) | A^{\chi'}_s = a]$ is also easy. This value is the mean of the portion of $l_a$ that is beneath all other lines. Figure \ref{drops}b illustrates these computations.

\begin{figure}[ht]
\centering
\includegraphics[scale = 0.45]{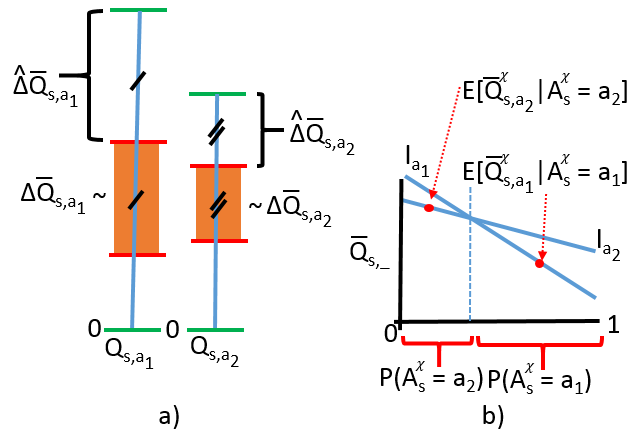}
\caption{a) Hypothetical drops in upper bounds on the Q-values of two actions, $a_1$ and $a_2$. We assume the next Q-value drop resulting from another cycle of thinking, $\Delta \overline{Q}$, is drawn from a range equal to the last drop from thinking, $\hat{\Delta} \overline{Q}$  b) Assuming perfect correlation in the speed of decrease in the Q-value upper bounds, as the upper bounds of the two actions drop from an additional cycle of thinking, initially $a_2$ has a better upper bound, but eventually $a_1$ overtakes $a_2$.}
\label{drops}
\end{figure}

Whether or not we make the assumption of action independence, we further speed up the computations by only calculating $E[\overline{Q}^{\chi'}(s', a) | A^{\chi'}_s = a]$ and $P(A^{\chi'}_{s'} = a)$ for the two ``most promising'' actions $a$, those with the lowest expectation of potential upper bounds. This limits the computation time to the time required to determine these actions (linear in $|A|$), and makes the time complexity of estimating $Q^*(s^m, \texttt{NOP})$ for one state $s$ be $O(K|A|)$ instead of $O(K|A|^2)$.



\subsubsection{Estimating $\mathbf{Q^*(s^m, f(s,\chi))}$}
Now that we have described how to estimate the value of taking the \texttt{NOP} action, we describe how to estimate the value of taking the currently recommended action, $f(s,\chi)$. We estimate $Q^*(s^m, f(s,\chi))$ by computing E[$\overline{Q}^{\chi'}(s, f(s, \chi))$], which takes constant time, keeping the overall time complexity linear. The reason we estimate $Q^*(s^m, f(s,\chi))$ using future Q-value upper bound estimates based on a probabilistic projection of $\chi'$, as opposed to our current Q-value upper bounds based on the current configuration $\chi$, is to make use of the more informed bounds derived at the future utility estimation. As the BRTDP algorithm is given more computation time, it can more accurately estimate the upper bound of a policy. This type of approximation has been justified before \cite{russell-aij91}. \comment{Estimating current values using future projections is not unprecedented and indeed justified \cite{russell-aij91}, as we should always use the best information we have to estimate the values of our actions. }In addition, using future utility estimates in both estimating $Q^*(s^m, f(s,\chi))$ and $Q^*(s^m, \texttt{NOP})$ provides a consistency guarantee: \emph{if thinking leads to no policy change, then our method estimates $VOC$ to be zero}. 

\comment{If $a = f(s,\chi)) \neq \texttt{NOP}$, the agent will not do any thinking in the current step.  By meta-myopicity, the agent will not do any thinking after acting either, so we approximate the Q-value of $a$ in the meta-MDP as $Q^*(s^m, f(s,\chi)) = Q^\chi(s, a)$ where $Q^\chi$ is the Q-value function of the policy corresponding to BRTDP's configuration $\chi$. BRTDP's bounds on the MDP's value function immediately give us bounds on $Q^*(s, a)$ in any state $s$. By BRTDP's properties, we know that $Q^*(s, a) \in [\underline{Q}(s,a), \overline{Q}(s, a)]$. While this interval can be large, especially when BRTDP has just started running, we can maintain a \emph{belief distribution} over it, denoted as $\mathcal{Q}(s, a)$. This, in turn, lets us approximate $Q^*(s, a)$, and hence $Q^*(s^m, f(s,\chi))$, with the expectation of the current such belief:

\begin{align}\label{eq:approxQa}
Q^*(s^m, f(s,\chi)) \approx E[\mathcal{Q}(s, a)]
\end{align}}

\subsection{Putting It All Together}
The core of our algorithms involves the computations we have described, in every state $s$ the agent visits on the way to the goal. In the experiments, we denote \texttt{UnCorr Metareasoner} as the metareasoner that assumes the actions are uncorrelated, and \texttt{Metareasoner} as the metareasoner that does not make this assumption. To complete the algorithms, we ensure that they decide the agent should think for another cycle if $\hat{\Delta}\overline{Q}_{s,a}$ isn't yet available for the agent's current world state $s$ (e.g., because BRTDP has never updated bounds for this state's Q-value so far), since $VOC$ computation is not possible without prior observations on $\hat{\Delta}\overline{Q}_{s,a}$. Crucially, all our estimates make metareasoning take time only linear in the number of actions, $O(K|A|)$, per visited state.

\comment{
\begin{algorithm}
\caption{General Metareasoner}
\label{algorithm}
\begin{algorithmic}
\State Compute $\hat{u}_{s,a} = \overline{Q}^{\chi}(s,a) - \hat{\Delta}\overline{Q}_{s,a}$ for all actions $a$
\State Let $\hat{A}$ be the set of 2 actions with the lowest $\hat{u}$.
\If {$\hat{\Delta} \overline{Q}_{s,a}$ does not exist for any action $a \in \hat{A}$}
\State \Return \texttt{NOP}
\EndIf
\State Approximate $\hat{V} \approx VOC(s^m) $
\If {$\hat{V} > 0$}
\State \Return \texttt{NOP}
\Else
\State \Return $f(s, \chi)$.
\EndIf 
\end{algorithmic}
\end{algorithm}
}

\section{Experiments}

\comment{
\begin{table*}[tb]
\vspace{-0.2in}
\small
\centering
\begin{tabular}{@{}llll|lllll@{}}
\toprule
                & Heuristic & OptimalBase & MG$^{UB}$ & Think*Act & Prob  & NoInfoThink & UnCorr MR & MR    \\ \midrule
Stochastic (Thinking)    & 1089      & 103.9        & 10.5      & 265.9     & 250.5 & 220.1       & 253.7     & \textbf{207.3} \\
Stochastic (Acting)    & 767.3     & 68.1         & 11.3      & 120.3     & 109.3 & 104.3       & 166.5     & \textbf{91.5}  \\
Traps           & 979       & 113.5        & 8.6       & 678.6     & 393.8 & 379.6       & 389.8     & \textbf{347.9} \\
DynamicNOP-1 & 251.4     & 66           & 3.8       & \textbf{168.6}     & 623.5 & 378.4       & 235.3     & 343.1 \\
DynamicNOP-2 & 119.4     & 66           & 1.8       & 156.4     & 99.2  & 131.7       & \textbf{80.2}      & \textbf{80}    \\ \bottomrule
\end{tabular}
\vspace{-0.1in}
\caption{\small On the right, performance summary for \texttt{Metareasoner} (MR) and \texttt{Uncorr Metareasoner} (UnCorr MR), vs. the baselines (Think*Act, Prob, NoInfoThink) across test domains. On the left, upper bounds of metareasoning gaps ($MG^{UB}$) for all test domains, defined as the ratio of the expected cost of the initial heuristic policy (Heuristic) and that of an optimal one (OptimalBase) at the initial state.
\vspace{-0.1in}}
\label{resultstable}
\end{table*}
}

We evaluate our metareasoning algorithms in several synthetic domains designed to reflect a wide variety of factors that could influence the value of metareasoning. Our goal is to demonstrate the ability of our algorithms to estimate the value of computation and adapt to a plethora of world conditions. The experiments are performed on four domains, all of which are built on a $100 \times 100$ grid world, where the agent can move between cells at each time step to get to the goal located in the upper right corner. To initialize the lower and upper bounds of BRTDP, we use the zero heuristic and an appropriately scaled (multiplied by a constant) Manhattan distance to the goal, respectively. 
\subsection{Domains}
The four world domains are as follows:
\begin{itemize}
\item \emph{Stochastic.} This domain adds winds to the grid world to be analogous to worlds with stochastic state transitions. Moving against the wind causes slower movement across the grid, whereas moving with the wind results in faster movement. 
The agent's initial state is the southeast corner and the goal is located in the northeast corner. We set the parameters of the domain as follows so that there is a policy that can get the agent to the goal with a small number of steps (in tens instead of hundreds) and where the winds significantly influence the number of steps needed to get to the goal: The agent can move 11 cells at a time and the wind has a pushing power of 10 cells. The next location of the agent is determined by adding the agent's vector and the wind's vector except when the agent decides to think (executes \texttt{NOP}), in which case it stays in the same position. Thus, the winds can never push the agent in the opposite direction of its intention. The prevailing wind direction over most of the grid is northerly, except for the column of cells containing the goal and starting position, where it is southerly. Note that this southerly wind direction makes the initial heuristic extremely suboptimal. To simulate stochastic state transitions, the winds have their prevailing direction in a given cell with 60\% probability; with 40\% probability they have a direction orthogonal to the prevailing one (20\% easterly and 20\% westerly). 

We perform a set of experiments on this simplest domain of the set, to observe the effect of different costs for thinking and acting on the behaviors of algorithms. We vary the cost of thinking and acting between 1 and 15. When we vary the cost of thinking, we fix the cost of acting at 11, and when we vary the cost of acting, we fix the cost of thinking at 1.

\item \emph{Traps.} This domain modifies the \emph{Stochastic} domain to resemble the setting where costs for thinking and acting are not constant among states. To simplify the parameter choices, we fix the cost of thinking and acting to be equal, respectively, to the agent's moving distance and wind strength. Thus, the cost of thinking is 10 and the cost of acting is 11. To vary the costs of thinking and acting between states, we make thinking and acting at the initial state extremely expensive at a cost of 100, about 10 times the cost of acting and thinking in the other states. Thus, the agent is forced to think outside its initial state in order to perform optimally.

\item \emph{DynamicNOP-1.} In the previous domains, executing a \texttt{NOP} does not change the agent's state. In this domain, thinking causes the agent to move in the direction of the wind, causing the agent to stochastically transition as a result of thinking. In this domain, the cost of thinking is composed of both explicit and implicit components; a static value of 1 unit and a dynamic component determined by stochastic state transitions as a result of thinking. The static value is set to 1 so that the dynamic component can dominate the decisions about thinking. The agent starts in cell $(98,1)$. We change the wind directions so that there are easterly winds in the most southern row and northerly winds in the most eastern row that can push the agent very quickly to the goal. Westerly winds exist everywhere else, pushing the agent away from the goal. We change the stochasticity of the winds so that the westerly winds change to northerly winds with 20\% probability, and all other wind directions are no longer stochastic. We lower the amount of stochasticity to better see if our agents can reason about the implicit costs of thinking. The wind directions are arranged so that there is potential for the agent to improve upon its initial policy but thinking is risky as it can move the agent to the left region, which is hard to recover from since all the winds push the agent away from the goal.

\item \emph{DynamicNOP-2.} This domain is just like the previous domain, but we change the direction of the winds in the northern-most row to be easterly. These winds also do not change directions. In this domain, as compared to the previous one, it is less risky to take a thinking action; even when the agent is pushed to the left region of the board, the agent can find strategies to get to the goal quickly by utilizing the easterly wind at the top region of the board. 
\end{itemize}

\comment{
We now describe the three domains at a high level, with diagrams and more detailed descriptions provided in the appendix. In the \emph{Stochastic} domain, the agent starts in the bottom right corner. Winds push it away from the goal if it tries reach the goal along the shortest Euclidean path, but push the agent towards the goal if it takes alternate routes. In this domain, metareasoning can help the agent discover paths to the goal that are more efficient than the one recommended by the (Manhattan) heuristic. The \emph{Traps} domain extends the \emph{Stochastic} domain by charging an extremely high cost for acting and thinking in the initial state. This domain tests if the algorithms can adapt its metareasoning behavior to variation in costs. In both of these domains, \texttt{NOP} doesn't carry the agent to another state, but its cost varies depending on the wind, intuitively reflecting the effort the agent must exert to stay in place. In the \emph{DynamicNOP} domain, the agent moves according to world dynamics when it chooses \texttt{NOP}.}

\subsection{The Metareasoning Gap}
\comment{
The value of approximate metareasoning differs across domains and their parameters (action costs, etc). E.g., in the vanilla grid world with no winds, the initial policy, implied by the Manhattan heuristic, is optimal, and optimal metareasoning policy is to never think. 
In contrast, in \emph{Stochastic}, metareasoning can lead to a much improved policy. }

We introduce the concept of the \emph{metareasoning gap} as a way to quantify the potential improvement over the initial heuristic-implied policy, denoted as \texttt{Heuristic}, that is possible due to optimal metareasoning. The metareasoning gap is the ratio of the expected cost of \texttt{Heuristic} for the base MDP to the expected cost of the optimal metareasoning policy, computed at the initial state.  Exact computation of the metareasoning gap requires evaluating the optimal metareasoning policy and is infeasible. Instead, we compute an upper bound on the metareasoning gap by substituting the cost of the optimal metareasoning policy with the cost of the optimal policy for the \emph{base} MDP (denoted \texttt{OptimalBase}). The metareasoning gap can be no larger than this upper bound, because metareasoning can only add cost to \texttt{OptimalBase}. We quantify each domain with this upper bound ($MG^{UB}$) in Table \ref{resultstable} and show that our algorithms for metareasoning provide significant benefits when $MG^{UB}$ is high. We note that none of the algorithms use the metareasoning gap in its reasoning. \\

\begin{table}[h]
\vspace{-0.2in}
\centering
\begin{tabular}{@{}llll@{}}
\toprule
                & Heuristic & OptimalBase & MG$^{UB}$   \\ \midrule
Stochastic (Thinking)    & 1089      & 103.9        & 10.5   \\
Stochastic (Acting)    & 767.3     & 68.1         & 11.3    \\
Traps           & 979       & 113.5        & 8.6   \\
DynamicNOP-1 & 251.4     & 66           & 3.8   \\
DynamicNOP-2 & 119.4     & 66           & 1.8  \\ \bottomrule
\end{tabular}
\caption{Upper bounds of metareasoning gaps ($MG^{UB}$) for all test domains, defined as the ratio of the expected cost of the initial heuristic policy (\texttt{Heuristic}) to that of an optimal one (\texttt{OptimalBase}) at the initial state.
\vspace{-0.1in}}
\label{resultstable}
\end{table}

\subsection{Experimental Setup}
We compare our metareasoning algorithms against a number of baselines. The \texttt{Think$^*$Act} baseline simply plans for $n$ cycles at the initial state and then executes the resulting policy, without planning again. We also consider the \texttt{Prob} baseline, which chooses to plan with probability $p$ at each state, and executes its current policy with probability $1-p$.  An important drawback of these baselines is that their performance is sensitive to their parameters $n$ and $p$, and the optimal parameter settings vary across domains. The \texttt{NoInfoThink} baseline plans for another cycle if it does not have information about how the BRTDP upper bounds will change. This baseline is a simplified version of our algorithms that does not try to estimate the $VOC$.

For each experimental condition, we run each metareasoning algorithm until it reaches the goal 1000 times and average the results to account for stochasticity. Each BRTDP trajectory is 50 actions long.

\subsection{Results}

\comment{
\blindtext

\begin{figure*}
\begin{floatrow}
\ffigbox[0.33\linewidth][]{%
\includegraphics[scale = 0.4]{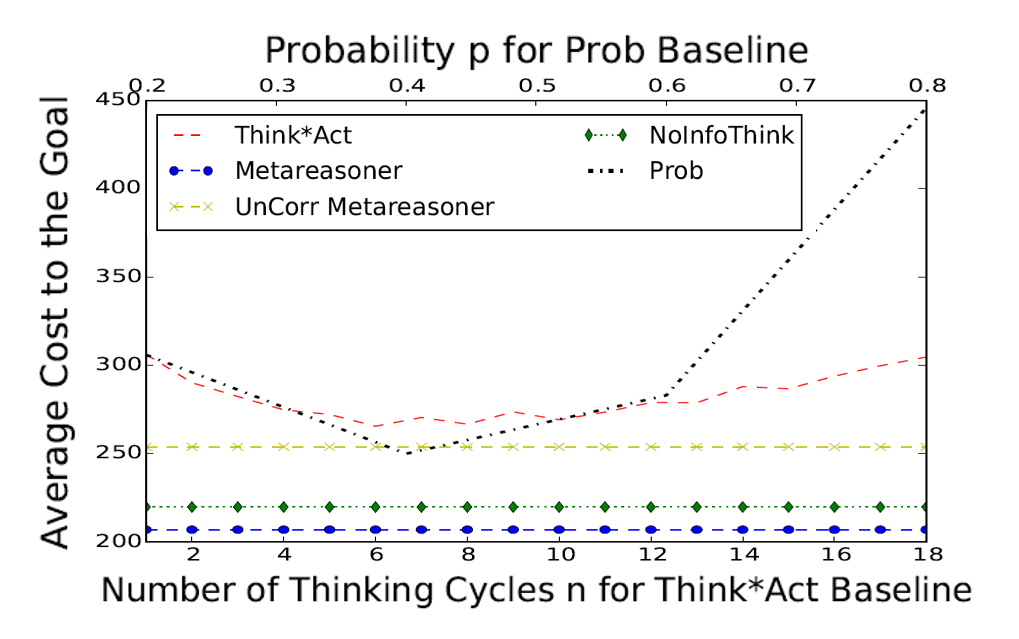}
\label{winds-dCOT-average}
}
{
  \caption{ Comparison of algorithms with baselines at the \emph{Stochastic} domain. }
}

\capbtabbox[0.66\linewidth][]{%
 
\small
\centering
\begin{tabular}{@{}llll|lllll@{}}
\toprule
                & Heuristic & OptimalBase & MG$^{UB}$ & Think*Act & Prob  & NoInfoThink & UnCorr MR & MR    \\ \midrule
Stochastic (Thinking)    & 1089      & 103.9        & 10.5      & 265.9     & 250.5 & 220.1       & 253.7     & 207.3 \\
Stochastic (Acting)    & 767.3     & 68.1         & 11.3      & 120.3     & 109.3 & 104.3       & 166.5     & 91.5  \\
Traps           & 979       & 113.5        & 8.6       & 678.6     & 393.8 & 379.6       & 389.8     & 347.9 \\
DynamicNOP 1 & 251.4     & 66           & 3.8       & 168.6     & 623.5 & 378.4       & 235.3     & 343.1 \\
DynamicNOP 2 & 119.4     & 66           & 1.8       & 156.4     & 99.2  & 131.7       & 80.2      & 80    \\ \bottomrule
\end{tabular}
\label{resultstable}

}
{%
  \caption{Summary of experimental results.}%
}
\end{floatrow}
\end{figure*}
}

In \emph{Stochastic}, we perform several experiments by varying the costs of thinking (\texttt{NOP}) and acting. We observe (figures can be found in the appendix) that when the cost of thinking is low or when the cost of acting is high, the baselines do well with high values of $n$ and $p$, and when the costs are reversed, smaller values do better. This trend is expected, since lower thinking cost affords more thinking, but these baselines don't allow for predicting the specific ``successful'' $n$ and $p$ values in advance. \texttt{Metareasoner} does not require parameter tuning and beats even the best performing baseline for all settings. Figure \ref{bigfigure}a compares the metareasoning algorithms against the baselines when the results are averaged over the various settings of the  cost of acting, and Figure \ref{bigfigure}b shows results averaged over the various settings of the  cost of thinking. \texttt{Metareasoner} does extremely well in this domain because the metareasoning gap is large, suggesting that metareasoning can improve the initial policy significantly. Importantly, we see that \texttt{Metareasoner} performs better than \texttt{NoInfoThink}, which shows the benefit from reasoning about how the bounds on Q-values will change.  \texttt{UnCorr Metareasoner} does not do as well as \texttt{Metareasoner}, probably because the assumption that actions' Q-values are uncorrelated does not hold well. 

\begin{figure*}[t]
\vspace{-0.2in}
\centering
\includegraphics[scale = 0.49]{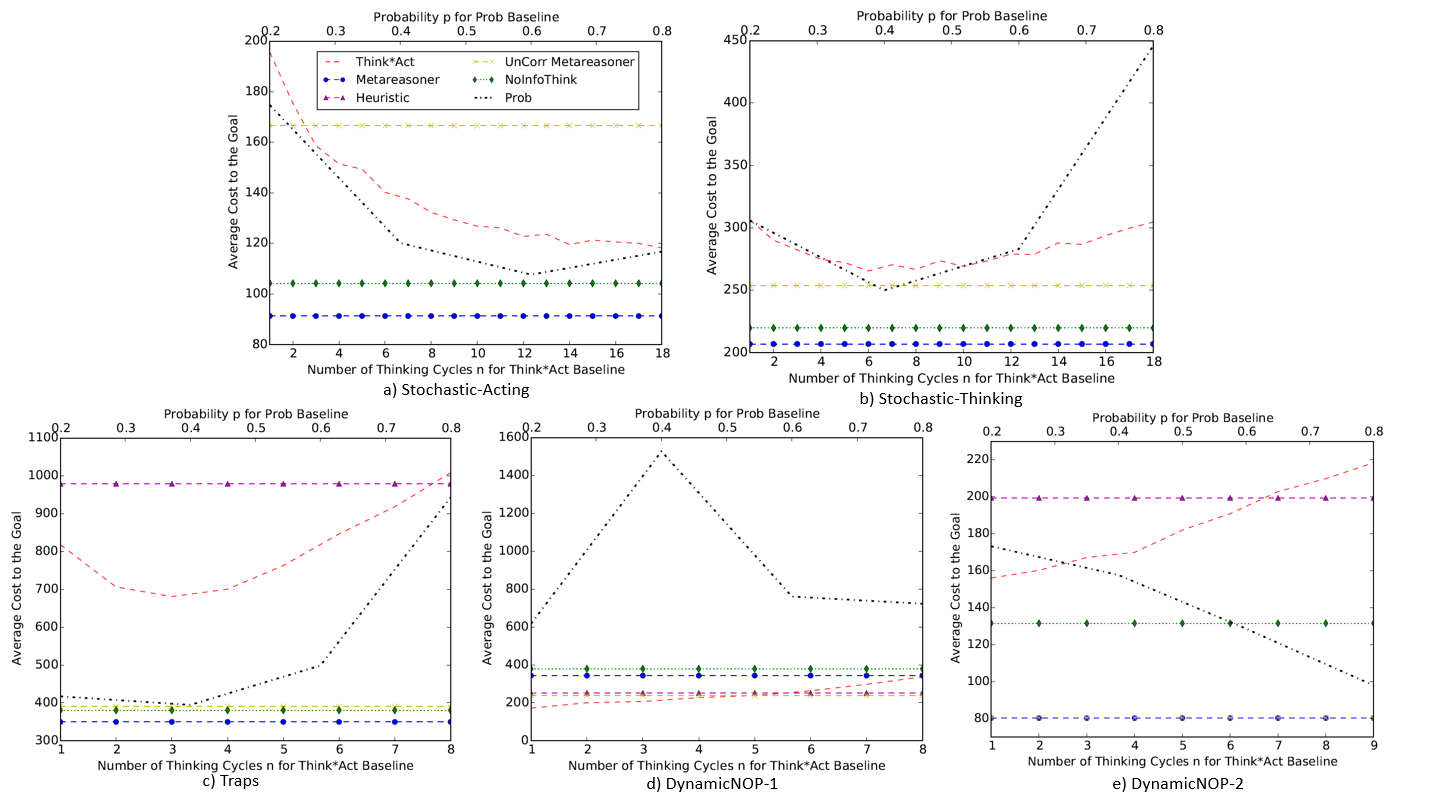}
\vspace{-0.3in}
\caption{Comparison of \texttt{Metareasoner} and \texttt{Uncorr Metareasoner} with baselines on experimental domains. Some figures do not include \texttt{Heuristic} when it performs especially poorly for readability.  
\vspace{-0.1in}
}
\label{bigfigure}
\end{figure*}

We now turn to \emph{Traps}, where thinking and acting in the initial state incurs significant cost. Figure \ref{bigfigure}c again summarizes the results. \texttt{Think$^*$Act} performs very poorly, because it is limited to thinking only at the initial state. \texttt{Metareasoner} does well, because it figures out that it should not think in the initial state (beyond the initial thinking step), and should instead quickly move to safer locations. \texttt{Uncorr Metareasoner} also closes the metareasoning gap significantly, but again not as much as \texttt{Metareasoner}.

We now consider \emph{DynamicNOP-1}, a domain adversarial to approximate metareasoning, because winds almost everywhere push the agent away from the goal. There are only a few locations from which winds can carry the agent to the destination. 
Figure \ref{bigfigure}d shows that our algorithms do not achieve large gains. However, this result is not surprising. The best policy involves little thinking, because whenever the agent chooses to think, it is pushed away from the goal, and  thinking for just a few consecutive time steps can take the agent to states where reaching the goal is extremely difficult. Consequently, \texttt{Think$^*$Act} with 1-3 thinking steps turns out to be near-optimal, since it is pushed away from the goal only slightly and can use a slightly improved heuristic to head back. \texttt{Metareasoner} actually does well in many individual runs, but occasionally thinks longer due to $VOC$ computation stochasticity and can get stuck, yielding higher average policy cost. 
In particular, it may frequently be pushed into a state that it has never encountered before, where it must think again because it does not have any history about how BRTDP's bounds have changed in that state, and then subsequently get pushed into an unencountered state again. In this domain, our approximate algorithms can diverge away from an optimal policy, which would plan very little to minimize the risk of being pushed away from the goal. 

\comment{If the agent is unable to quickly discover one of a few ``jetstreams" that can quickly carry it to the goal, it will be trapped. As we would expect, the metareasoning gap of this domain is quite small in comparison to the previous domains, because the initial heuristic does not fall into such traps.}

\emph{DynamicNOP-2} 
provides the agent more opportunities to recover even if it makes a poor decision. Figure \ref{bigfigure}e demonstrates that our algorithms perform much better in \emph{DynamicNOP-2} than in \emph{DynamicNOP-1}.  In \emph{DynamicNOP-2}, even if our algorithms do not discover the jetstreams that can push it towards the goal from initial thinking, they are provided more chances to recover when they get stuck. When thinking can move the agent on the board, having more opportunities to recover reduces the risk associated with making suboptimal thinking decisions.  Interestingly, the metareasoning gap is decreased at the initial state by the addition of the extra jetstream. However, the metareasoning gap at many other states in the domain is increased, showing that the metareasoning gap at the initial state is not the most ideal way to characterize the potential for improvement via metareasoning in all domains.

\section{Conclusion and Future Work}
We formalize and analyze the general metareasoning problem for MDPs, demonstrating that metareasoning is only polynomially harder than solving the base MDP. Given the determination that optimal general metareasoning is impractical, we turn to approximate metareasoning algorithms, which estimate the value of computation by relying on bounds given by BRTDP. Finally, we empirically compare our metareasoning algorithms to several baselines on problems designed to reflect challenges posed across a spectrum of worlds, and show that the proposed algorithms are much better at closing large metareasoning gaps. 

We have assumed that the agent can plan only when it takes the \texttt{NOP} action. A generalization of our work would allow varying amounts of thinking as part of any action. Some actions may consume more CPU resources than others, and actions which do not consume all resources during execution can allocate the remainder to planning. We also can relax the meta-myopic assumption, so as to consider the consequences of thinking for more than one cycle. In many cases, assuming that the agent will only think for one more step can lead to underestimation of the value of thinking, since many cycles of thinking may be necessary to see significant value. This ability can be obtained with our current framework by projecting changes in bounds for multiple steps. However, in experiments to date, we have found that pushing out the horizon of analysis was associated with large accumulations of errors and poor performance due to approximation upon approximation from predictions about multiple thinking cycles.  Finally, we may be able to improve our metareasoners by learning about and harnessing more details about the base-level planner. In our \texttt{Metareasoner} approximation scheme, we make strong assumptions about how the upper bounds provided by BRTDP will change, but learning distributions over these changes may improve performance. More informed models may lead to accurate estimation of non-myopic value of computation. However, learning distributions in a domain-independent manner is difficult, since the planner's behavior is heavily dependent on the domain and heuristic at hand.

\bibliographystyle{named}
\bibliography{main}

\begin{thebibliography}{}

\bibitem[\protect\citeauthoryear{Abelson \bgroup \em et al.\egroup
  }{1985}]{abelson-et-al:scheme}
Harold Abelson, Gerald~Jay Sussman, and Julie Sussman.
\newblock {\em Structure and Interpretation of Computer Programs}.
\newblock MIT Press, Cambridge, Massachusetts, 1985.

\bibitem[\protect\citeauthoryear{Baumgartner \bgroup \em et al.\egroup
  }{2001}]{bgf:Lixto}
Robert Baumgartner, Georg Gottlob, and Sergio Flesca.
\newblock Visual information extraction with {Lixto}.
\newblock In {\em Proceedings of the 27th International Conference on Very
  Large Databases}, pages 119--128, Rome, Italy, September 2001. Morgan
  Kaufmann.

\bibitem[\protect\citeauthoryear{Brachman and
  Schmolze}{1985}]{brachman-schmolze:kl-one}
Ronald~J. Brachman and James~G. Schmolze.
\newblock An overview of the {KL-ONE} knowledge representation system.
\newblock {\em Cognitive Science}, 9(2):171--216, April--June 1985.

\bibitem[\protect\citeauthoryear{Gottlob \bgroup \em et al.\egroup
  }{2002}]{gls:hypertrees}
Georg Gottlob, Nicola Leone, and Francesco Scarcello.
\newblock Hypertree decompositions and tractable queries.
\newblock {\em Journal of Computer and System Sciences}, 64(3):579--627, May
  2002.

\bibitem[\protect\citeauthoryear{Gottlob}{1992}]{gottlob:nonmon}
Georg Gottlob.
\newblock Complexity results for nonmonotonic logics.
\newblock {\em Journal of Logic and Computation}, 2(3):397--425, June 1992.

\bibitem[\protect\citeauthoryear{Levesque}{1984a}]{levesque:functional-foundations}
Hector~J. Levesque.
\newblock Foundations of a functional approach to knowledge representation.
\newblock {\em Artificial Intelligence}, 23(2):155--212, July 1984.

\bibitem[\protect\citeauthoryear{Levesque}{1984b}]{levesque:belief}
Hector~J. Levesque.
\newblock A logic of implicit and explicit belief.
\newblock In {\em Proceedings of the Fourth National Conference on Artificial
  Intelligence}, pages 198--202, Austin, Texas, August 1984. American
  Association for Artificial Intelligence.

\bibitem[\protect\citeauthoryear{Nebel}{2000}]{nebel:jair-2000}
Bernhard Nebel.
\newblock On the compilability and expressive power of propositional planning
  formalisms.
\newblock {\em Journal of Artificial Intelligence Research}, 12:271--315, 2000.

\end{thebibliography}

\begin{thebibliography}{}

\bibitem[\protect\citeauthoryear{Barto \bgroup \em et al.\egroup
  }{1995}]{barto-aij95}
Andrew~G. Barto, Steven~J. Bradtke, and Satinder~P. Singh.
\newblock Learning to act using real-time dynamic programming.
\newblock {\em Artificial Intelligence}, 72:81--138, 1995.

\bibitem[\protect\citeauthoryear{Bertsekas and
  Tsitsiklis}{1996}]{bertsekas-book}
Dimitri~P. Bertsekas and John Tsitsiklis.
\newblock {\em Neuro-dynamic Programming}.
\newblock Athena Scientific, 1996.

\bibitem[\protect\citeauthoryear{Breese and Horvitz}{1990}]{breese-uai90}
John~S. Breese and Eric Horvitz.
\newblock Ideal reformulation of belief networks.
\newblock In {\em UAI}, 1990.

\bibitem[\protect\citeauthoryear{Burns \bgroup \em et al.\egroup
  }{2013}]{burns-jair13}
Ethan Burns, Wheeler Ruml, and Minh~B. Do.
\newblock Heuristic search when time matters.
\newblock {\em Journal of Artificial Intelligence Research}, 47:697--740, 2013.

\bibitem[\protect\citeauthoryear{Chanel \bgroup \em et al.\egroup
  }{2014}]{chanel-icaps14}
Caroline P.~Carvalho Chanel, Charles Lesire, and Florent
  Teichteil-K{\"o}nigsbuch.
\newblock A robotic execution framework for online probabilistic (re)planning.
\newblock In {\em ICAPS}, 2014.

\bibitem[\protect\citeauthoryear{Dean \bgroup \em et al.\egroup
  }{1995}]{dean-aij95}
Thomas Dean, Leslie~Pack Kaelbling, Jak Kirman, and Ann Nicholson.
\newblock Planning under time constraints in stochastic domains.
\newblock {\em Artificial Intelligence}, 76:35--74, 1995.

\bibitem[\protect\citeauthoryear{Hansen and Zilberstein}{2001}]{hansen-aij01}
Eric~A Hansen and Shlomo Zilberstein.
\newblock Monitoring and control of anytime algorithms: A dynamic programming
  approach.
\newblock {\em Artificial Intelligence}, 126(1):139--157, 2001.

\bibitem[\protect\citeauthoryear{Hay \bgroup \em et al.\egroup
  }{2012}]{hay-uai12}
Nick Hay, Stuart Russell, David Toplin, and Solomon~Eyal Shimony.
\newblock Selecting computations: Theory and applications.
\newblock In {\em UAI}, 2012.

\bibitem[\protect\citeauthoryear{Horvitz and Breese}{1990}]{horvitz-tr90}
Eric~J. Horvitz and John~S. Breese.
\newblock Ideal partition of resources for metareasoning.
\newblock Technical Report KSL-90-26, Stanford University, 1990.

\bibitem[\protect\citeauthoryear{Horvitz and Klein}{1995}]{horvitz-uai95}
Eric Horvitz and Adrian Klein.
\newblock Reasoning, metareasoning, and mathematical truth: Studies of theorem
  proving under limited resources.
\newblock In {\em UAI}, 1995.

\bibitem[\protect\citeauthoryear{Horvitz \bgroup \em et al.\egroup
  }{1989}]{horvitz-ijcai89}
Eric~J. Horvitz, Gregory~F. Cooper, and David E.Heckerman.
\newblock Reflection and action under scarce resources: Theoretical principles
  and empirical study.
\newblock In {\em IJCAI}, 1989.

\bibitem[\protect\citeauthoryear{Horvitz \bgroup \em et al.\egroup
  }{2001}]{horvitz-uai01}
Eric Horvitz, Yongshao Ruan, Carla~P. Gomes, Henry Kautz, Bart Selman, and
  David~M. Chickering.
\newblock A bayesian approach to tackling hard computational problems.
\newblock In {\em UAI}, 2001.

\bibitem[\protect\citeauthoryear{Horvitz}{1987}]{horvitz-uai87}
Eric Horvitz.
\newblock Reasoning about beliefs and actions under computational resource
  constraints.
\newblock In {\em UAI}, 1987.

\bibitem[\protect\citeauthoryear{Horvitz}{2001}]{horvitz-aij01}
Eric Horvitz.
\newblock Principles and applications of continual computation.
\newblock {\em Artificial Intelligence}, 126:159--196, 2001.

\bibitem[\protect\citeauthoryear{Howard}{1960}]{howard-60}
R.A. Howard.
\newblock {\em Dynamic Programming and {M}arkov Processes}.
\newblock MIT Press, 1960.

\bibitem[\protect\citeauthoryear{Kautz \bgroup \em et al.\egroup
  }{2002}]{kautz-aaai02}
Henry Kautz, Eric Horvitz, Yongshao Ruan, Carla Gomes, and Bart Selman.
\newblock Dynamic restart policies.
\newblock In {\em AAAI}, 2002.

\bibitem[\protect\citeauthoryear{Keller and Gei{\ss}er}{2015}]{keller-aaai15}
Thomas Keller and Florian Gei{\ss}er.
\newblock Better be lucky than good: Exceeding expectations in mdp evaluation.
\newblock In {\em AAAI}, 2015.

\bibitem[\protect\citeauthoryear{Kocsis and
  Szepesv\'{a}ri}{2006}]{kocsis-ecml06}
Levente Kocsis and Csaba Szepesv\'{a}ri.
\newblock Bandit based monte-carlo planning.
\newblock In {\em ECML}, 2006.

\bibitem[\protect\citeauthoryear{Kolobov \bgroup \em et al.\egroup
  }{2012}]{kolobov-aaai12}
Andrey Kolobov, Mausam, and Daniel~S. Weld.
\newblock Lrtdp vs uct for online probabilistic planning.
\newblock In {\em AAAI}, 2012.

\bibitem[\protect\citeauthoryear{McMahan \bgroup \em et al.\egroup
  }{2005}]{mcmahan-icml05}
H.~Brendan McMahan, Maxim Likhachev, and Geoffrey~J. Gordon.
\newblock Bounded real-time dynamic programming: Rtdp with monotone upper
  bounds and performance guarantees.
\newblock In {\em ICML}, 2005.

\bibitem[\protect\citeauthoryear{Russell and Wefald}{1991}]{russell-aij91}
Stuart Russell and Eric Wefald.
\newblock Principles of metareasoning.
\newblock {\em Artificial intelligence}, 49(1):361--395, 1991.

\bibitem[\protect\citeauthoryear{Shahaf and Horvitz}{2009}]{shahaf-ijcai09}
Dafna Shahaf and Eric Horvitz.
\newblock Ivestigations of continual computation.
\newblock In {\em IJCAI}, 2009.

\bibitem[\protect\citeauthoryear{Sutton and Barto}{1998}]{sutton-98}
Richard~S. Sutton and Andrew~G. Barto.
\newblock {\em Introduction to Reinforcement Learning}.
\newblock MIT Press, 1998.

\bibitem[\protect\citeauthoryear{Zilberstein and
  Russell}{1993}]{zilberstein-ijcai93}
Shlomo Zilberstein and Stuart~J. Russell.
\newblock Anytime sensing, planning and action: A practical model for robot
  control.
\newblock In {\em IJCAI}, 1993.

\bibitem[\protect\citeauthoryear{Zilberstein and
  Russell}{1996}]{zilberstein-aij96}
Shlomo Zilberstein and Stuart Russell.
\newblock Optimal composition of real-time systems.
\newblock {\em Artificial Intelligence}, 82:181--213, 1996.

\end{thebibliography}
\newpage
\appendix
\section{Appendix}
\subsection{Proof of Theorem 3}
\begin{figure}
\centering
\includegraphics[scale = 0.2]{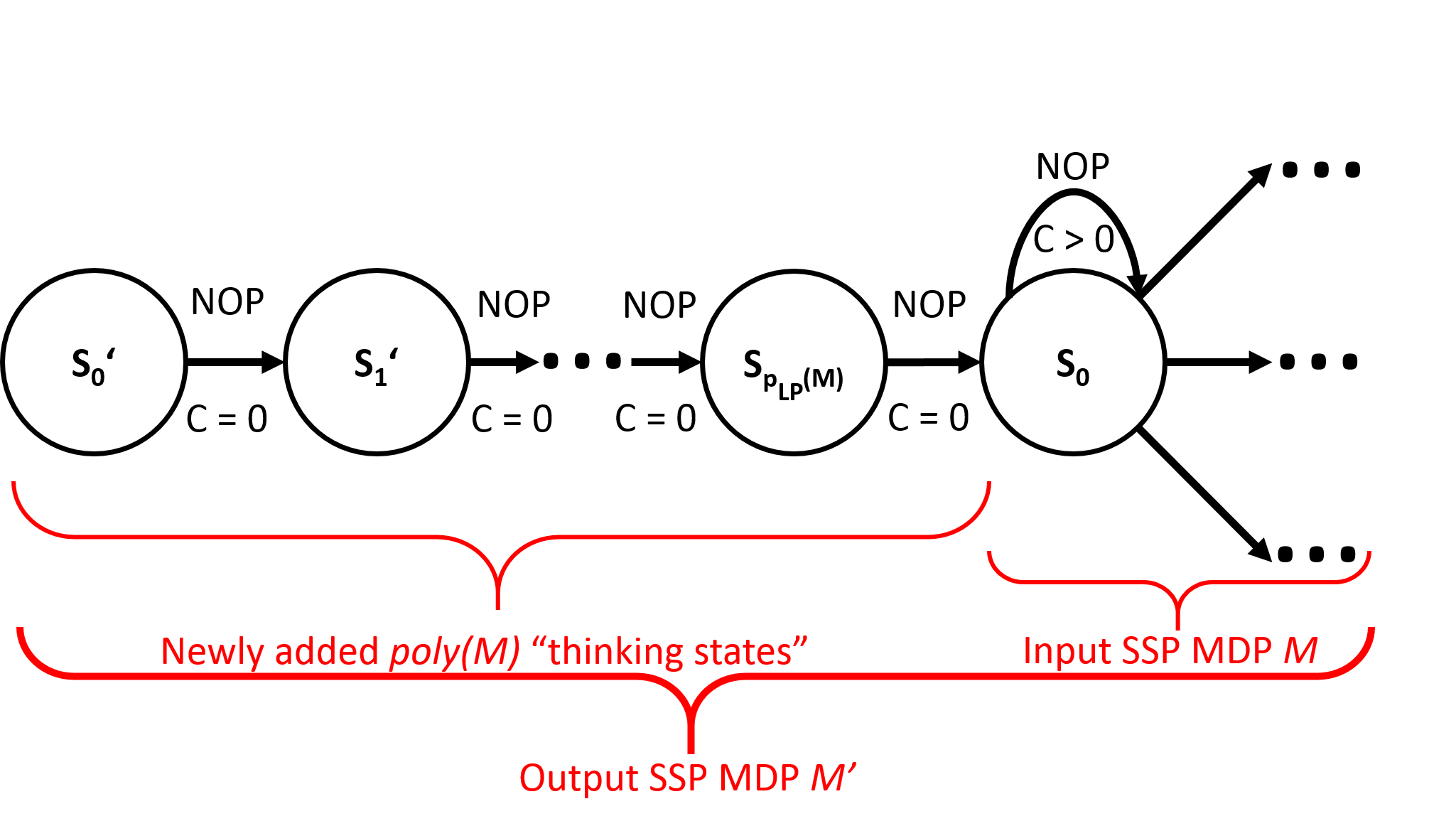}
\caption{The construction of a base MDP $M'$ for a metareasoning problem from an input SSP MDP $M$. }
\label{f:constr}
\end{figure}
\begin{proof}
By calling metareasoning $P$-complete we mean that there exists a Turing machine $B$ s.t. (1) for any input SSP MDP $M'$, $\texttt{Meta}_{B}(M')$ can be decided in time polynomial in $|M'|$, i.e., $\texttt{Meta}_{B}(M)$ is in $P$, and (2) there is a class of $P$-complete problems that can be converted to $\texttt{Meta}_{B}(M')$ via an \emph{NC}-reduction, i.e., by constructing $M'$ appropriately using a polynomial number of parallel Turing machines, each operating in polylogarithmic time.

The first part of the above claim follows from Theorem 2
: since SSP MDPs are solvable optimally by linear programming in polynomial time
, $\texttt{Meta}_{B}(M)$ is in $P$ if $B$ encodes a polynomial solver for linear programs.

For the second part, we perform an \emph{NC}-reduction from the class of SSP MDPs to the class of SSP MDPs-based metareasoning problems with respect to a fixed optimal polynomial-time solver $B$. Specifically, given an SSP MDP $M$ with an initial state, we show how to construct \emph{another} SSP MDP $M'$ s.t., for the optimal polynomial-time solver $B$ we describe shortly, deciding $\texttt{Meta}_{B}(M')$ is equivalent to deciding $M$. 

The intuition behind converting a given SSP MDP $M$ into $M'$, the SSP MDP that will serve as the base in our metareasoning problem, is to augment $M$ with new states \emph{where the agent can ``think'' by using a zero-cost \texttt{NOP} action} until the agent arrives at an optimal policy for the original states of $M$. Afterwards, the agent can transition from any of these newly added ``thinking states'' to $M$'s original start state $s_0$ and execute the optimal policy from there. Unfortunately, the proof is not as straightforward as it seems, because we cannot simply build $M$' by equipping $M$ with a new start state $s'_0$ with a self-loop zero-cost \texttt{NOP} action --- $M'$ with such an action would violate the SSP MDP definition. Below, we show how to overcome this difficulty. Since thinking in the newly added states of $M'$ costs nothing, the cost of an optimal policy for $\texttt{Meta}_{B}(M')$ is the same as for $M$, so deciding the former problem decides the latter.

The construction of $M'$ from a given SSP MDP $M$ is illustrated in Figure \ref{f:constr}. Consider the number of instruction-steps it takes to solve $M$ by linear programming. This number is polynomial; namely, there exists a polynomial $p_{LP}(|M|)$ that bounds $M$'s solution time from above. To transform $M$ into $M'$, we add a set of $p_{LP}(|M|)$ states, $s'_0, s'_1, \ldots, s'_{p_{LP}(|M|)}$ to $M$. These new states connect into a chain via zero-cost \texttt{NOP} actions: the start state $s'_0$ of $M'$ links to $s'_1$, $s'_1$ links to $s'_2$, and so on until $s'_{p_{LP}(|M|)}$ links to $s_0$, the start state of $M$. In addition, for all original states of $M$, we create a self-loop \texttt{NOP} action with a positive cost. The entire transformation can be easily implemented as an NC-reduction on $p_{LP}(|M|) + |S|$ computers, each recording the cost and transition function of \texttt{NOP} for a separate state. Since for each state, \texttt{NOP}'s cost and transition functions together can be encoded by just two numbers (\texttt{NOP} transition function assigns probability 1 to a single transition that is implicitly but unambiguously determined for every state), each computer operates in polylogarithmic time. Moreover, initializing each of the parallel machines with the MDP state for which it is supposed to write out the transition and cost function values is as simple as appropriately setting a pointer to the input tape, and can be done in log-space. Thus, the above procedure is a valid NC-reduction. Note also that $M'$ is an SSP MDP: although it has zero-cost actions, they do not form loops/strongly connected components.

Our motivation for constructing $M'$ as above was to provide an agent with enough states where it can ``think'' to guarantee that if the agent starts at $s'_0$,  it arrives at $M$'s initial state $s_0$ with a computed optimal policy from $s_0$ onwards. This would imply that the expected cost of an optimal policy for $\texttt{Meta}_{B}(M')$ from $s_0$ would be the same as for $M$. However, for this guarantee to hold, we need a general SSP MDP solver $B$ that can solve/decide $M'$ in time $O(poly(|M|))$, \emph{not} $O(poly(|M'|)$. The difference between $O(poly(|M'|))$ and $O(poly(|M|))$ is very important, because $M'$ is larger than $M$, so the newly added chain of states may not be enough for a $O(poly(|M'|))$ policy computation to have zero cost.  

To circumvent this issue, we define $B$ that recognizes ``lollypop-shaped'' MDPs $M'$ as in Figure \ref{f:constr}, which have an arbitrarily connected subset $S^c$ of the state space representing a sub-MDP $M^c$ preceded by a chain of \texttt{NOP}-connected states of size $p_{LP}(|M^c|)$ leading to $M^c$'s start state $s_0$, and ignores the linear chain part. (Note that the policy for the linear chain part is determined uniquely and there is no need to write it out explicitly).  For that, we assume the metareasoning problem's input SSP MDP to be in the form of a string

$$\mbox{$M^c$\_description\#\#\#chain\_description}$$

In this string, $\mbox{$M^c$\_description}$ stands for the arbitrarily connected part of the input MDP, and $\mbox{chain\_description}$ stands for the description of the linear \texttt{NOP}-connected chain. For MDPs violating conditions in Figure \ref{f:constr} (i.e., having a different connectivity structure or having the linear part of the wrong size), $\mbox{chain\_description}$ must be empty, with the entire MDP description placed before ``\#\#\#''. $B$ is defined to read that input string only up to ``\#\#\#'' and solve that part using LP.

Constructing $M'$ from $M$ and recording $M'$ in the aforementioned way ensures that the optimal policy for the metareasoning problem $\texttt{Meta}_{B}(M')$ chooses \texttt{NOP} until the agent reaches $M$'s start state $s_0$, by which point the agent will have computed an optimal policy for $M$. Coupled with the fact that $\texttt{Meta}_{B}(M')$ is in $P$, this implies the theorem's claim.

\end{proof}

\subsection{More Figures}
Figures \ref{winds-dCOT-1} through \ref{winds-dCOA-15} show results for the \emph{Stochastic} domain where we vary the cost of thinking and the cost of acting.

\begin{figure}
\centering
\includegraphics[scale = 0.45]{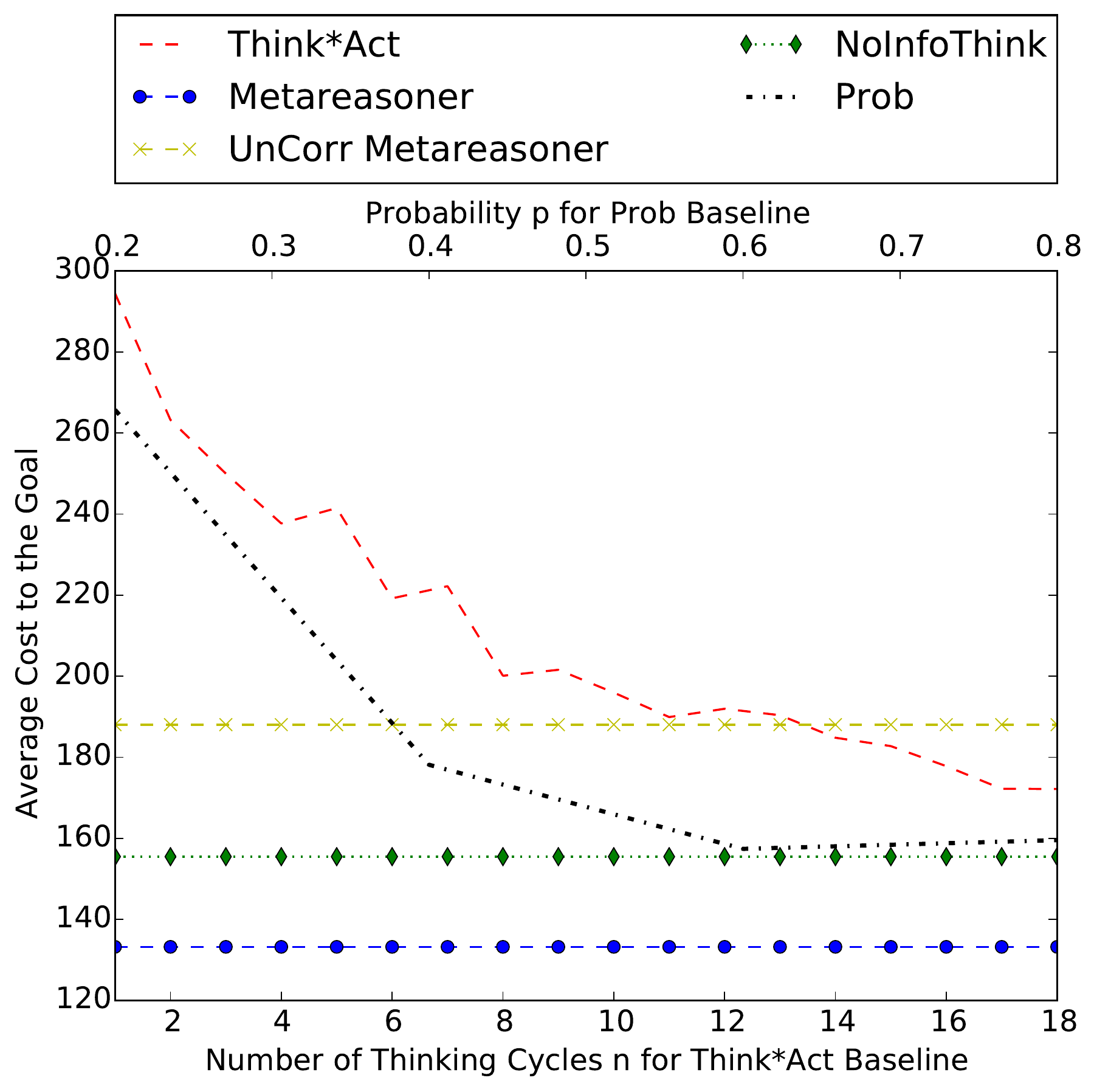}
\caption{Comparison of algorithms in \emph{Stochastic}, with the cost of thinking = 1 and cost of acting = 11}
\label{winds-dCOT-1}
\end{figure}
\begin{figure}
\centering
\includegraphics[scale = 0.45]{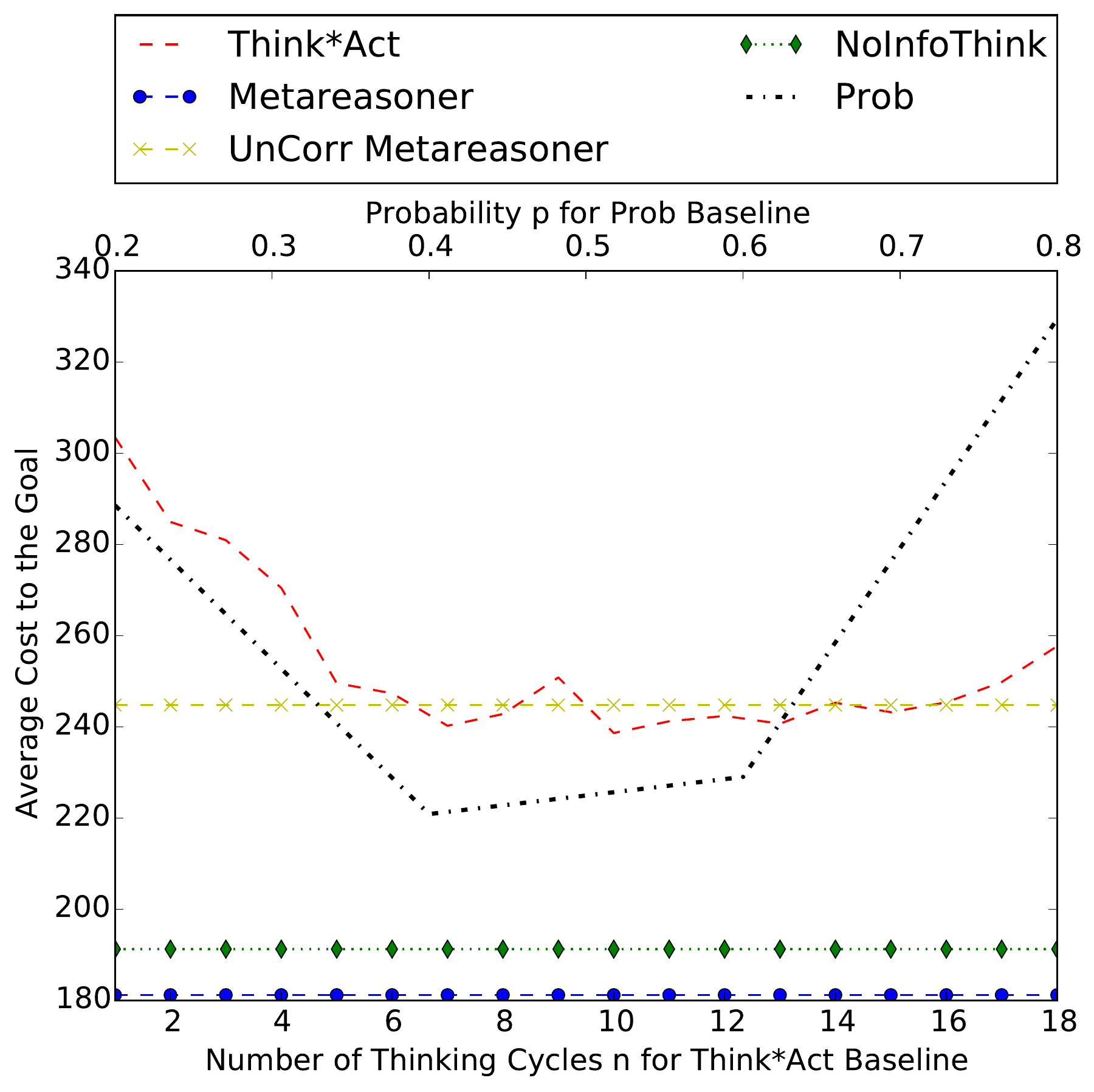}
\caption{Comparison of algorithms in \emph{Stochastic}, with the cost of thinking = 5 and cost of acting = 11}
\label{winds-dCOT-5}
\end{figure}
\begin{figure}
\centering
\includegraphics[scale = 0.45]{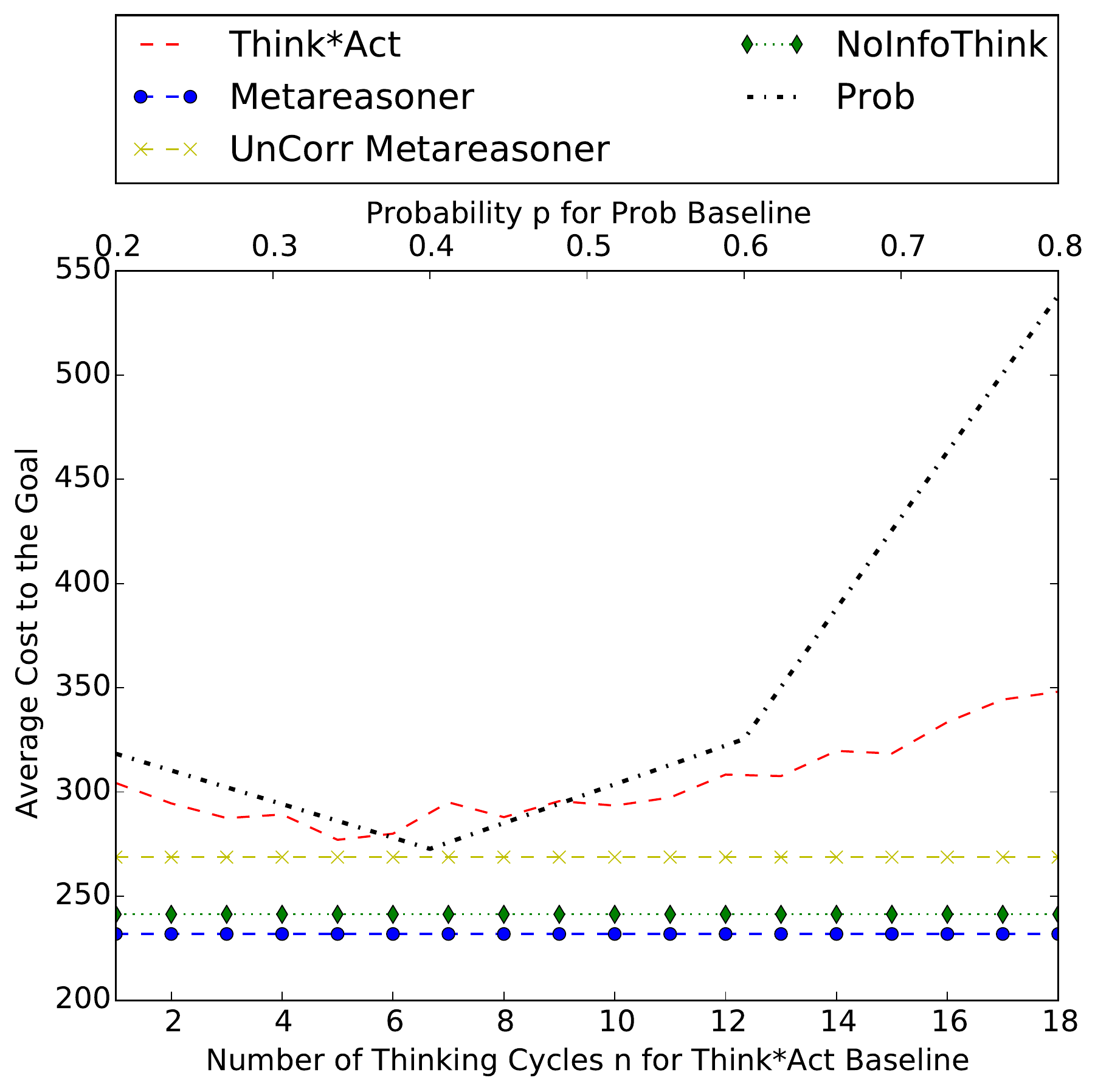}
\caption{Comparison of algorithms in \emph{Stochastic}, with the cost of thinking =10 and cost of acting = 11}
\label{winds-dCOT-10}
\end{figure}
\begin{figure}
\centering
\includegraphics[scale = 0.45]{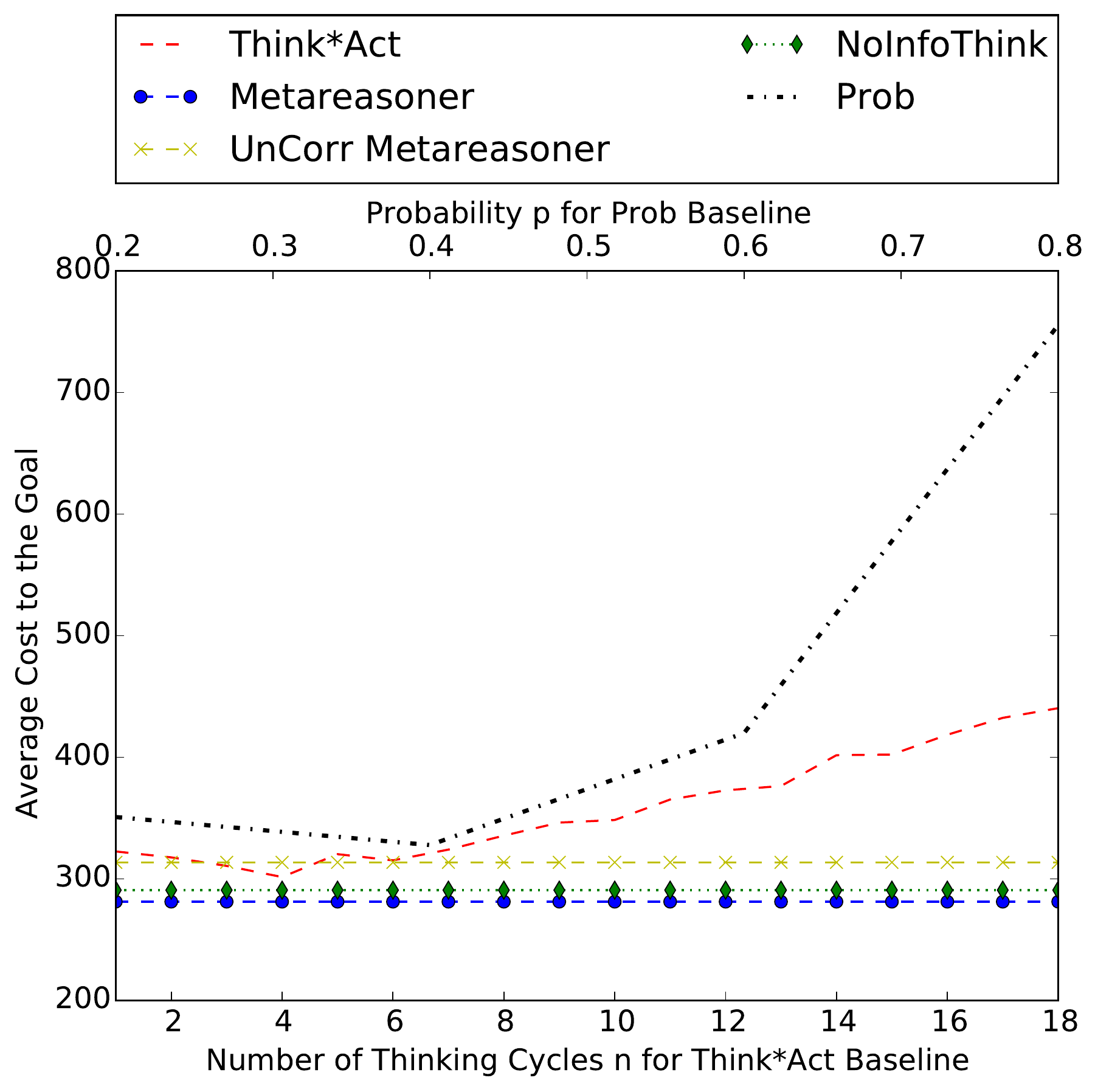}
\caption{Comparison of algorithms in \emph{Stochastic}, with the cost of thinking = 15 and cost of acting = 11}
\label{winds-dCOT-15}
\end{figure}
\begin{figure}
\centering
\includegraphics[scale = 0.45]{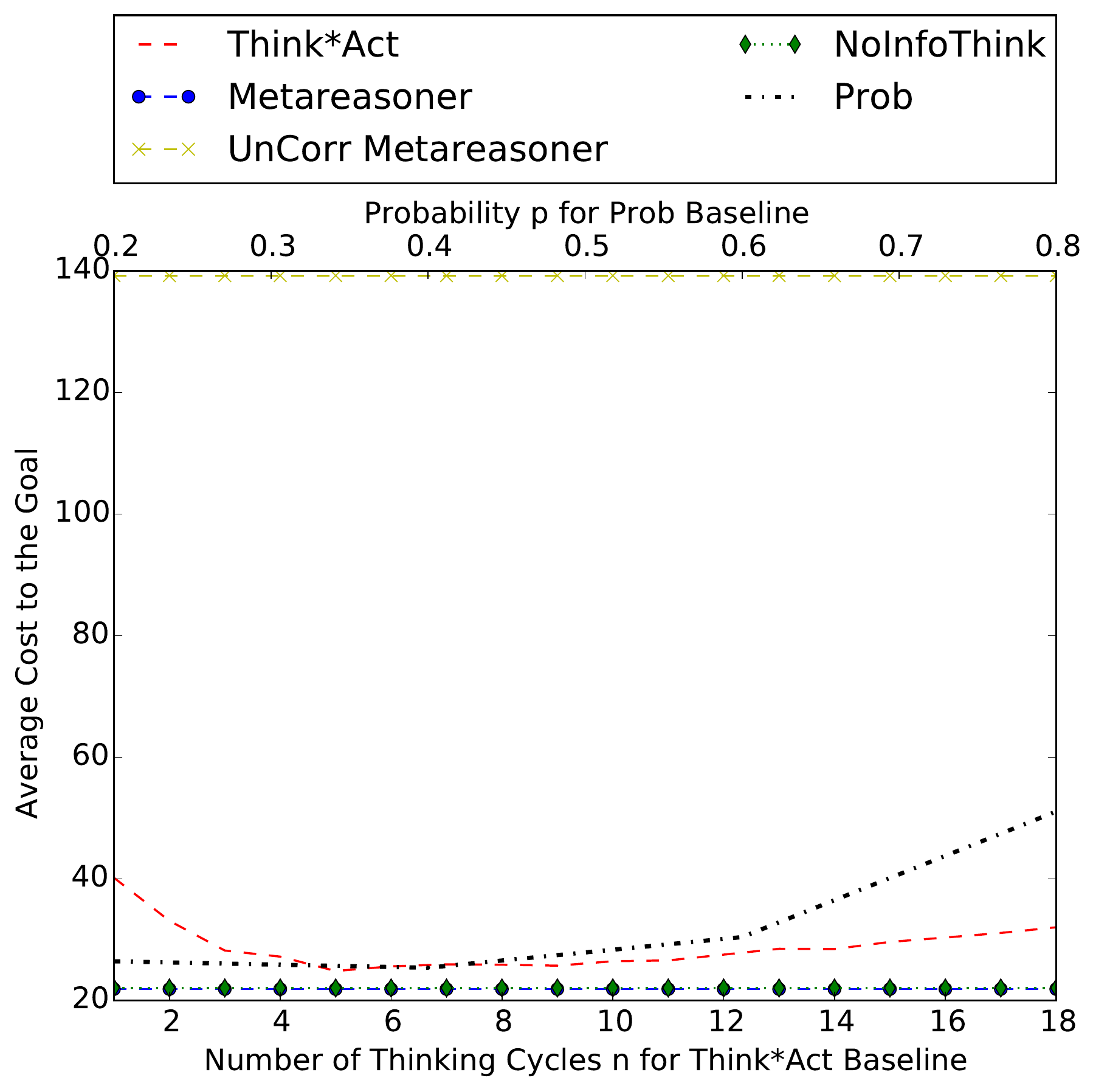}
\caption{Comparison of algorithms in \emph{Stochastic}, with the cost of acting = 1 and cost of thinking = 1}
\label{winds-dCOA-1}
\end{figure}
\begin{figure}
\centering
\includegraphics[scale = 0.45]{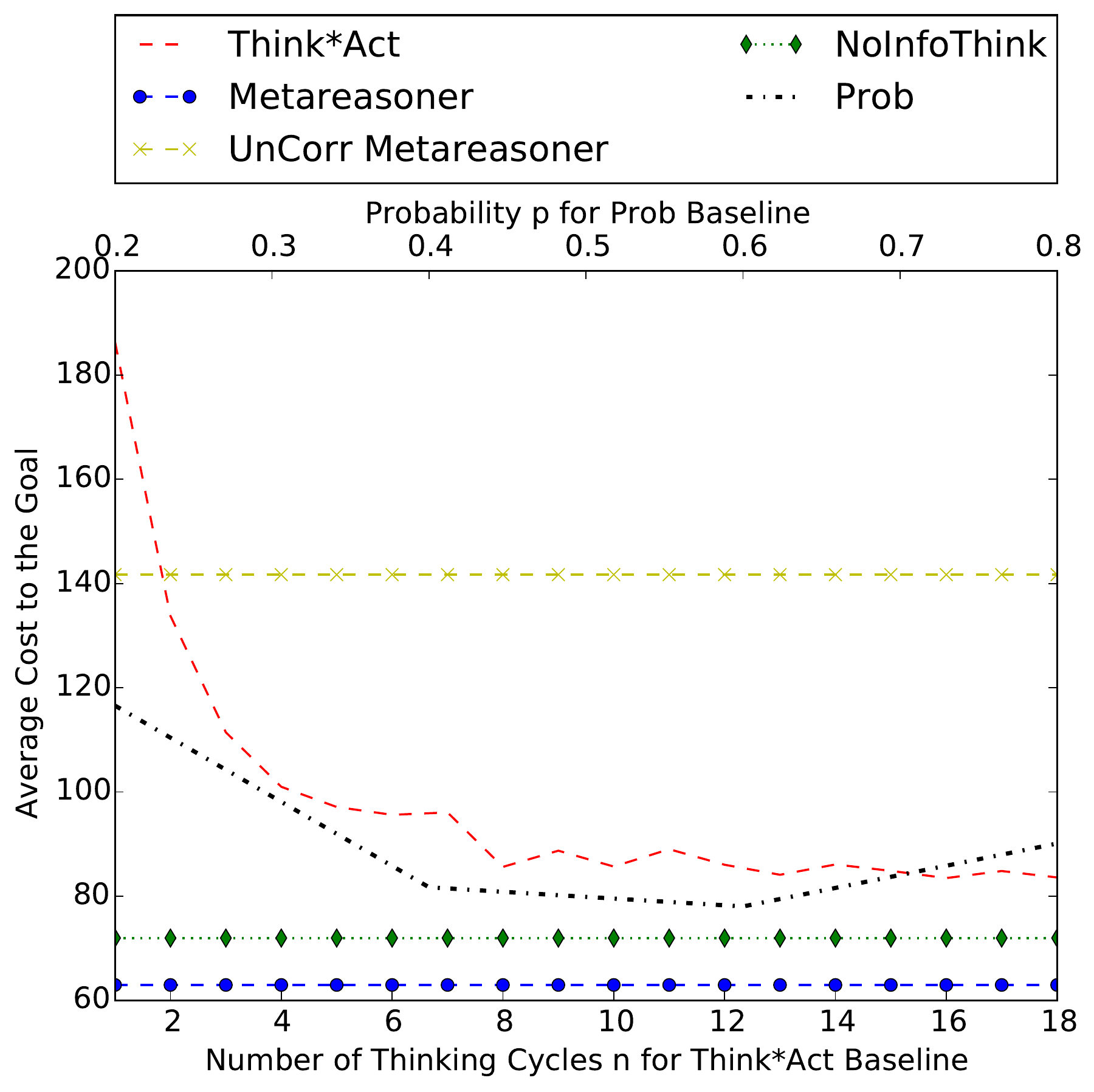}
\caption{Comparison of algorithms in \emph{Stochastic}, with the cost of acting = 5 and cost of thinking = 1}
\label{winds-dCOA-5}
\end{figure}
\begin{figure}
\centering
\includegraphics[scale = 0.45]{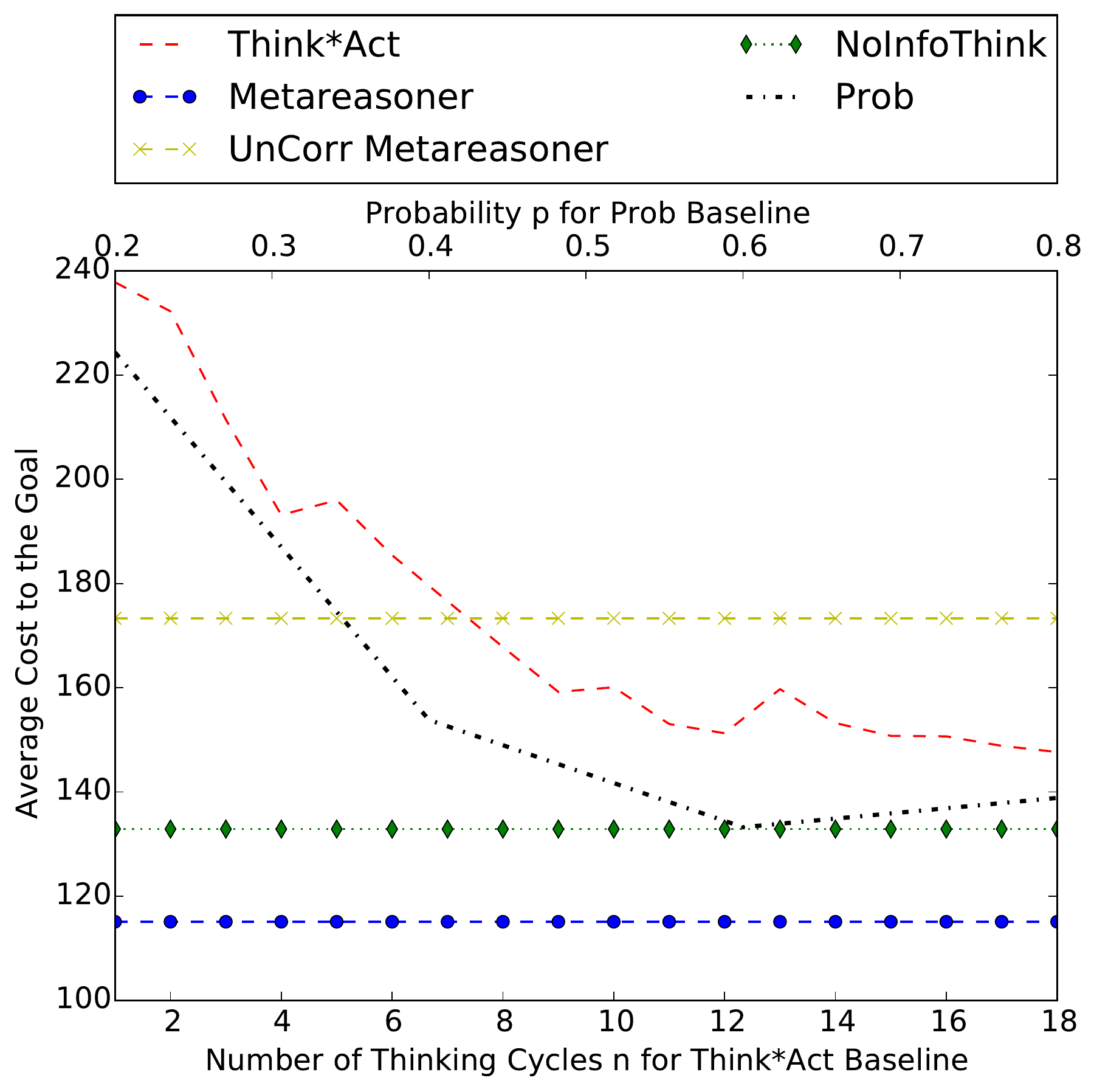}
\caption{Comparison of algorithms in \emph{Stochastic}, with the cost of acting = 10 and cost of thinking = 1}
\label{winds-dCOA-10}
\end{figure}
\begin{figure}
\centering
\includegraphics[scale = 0.45]{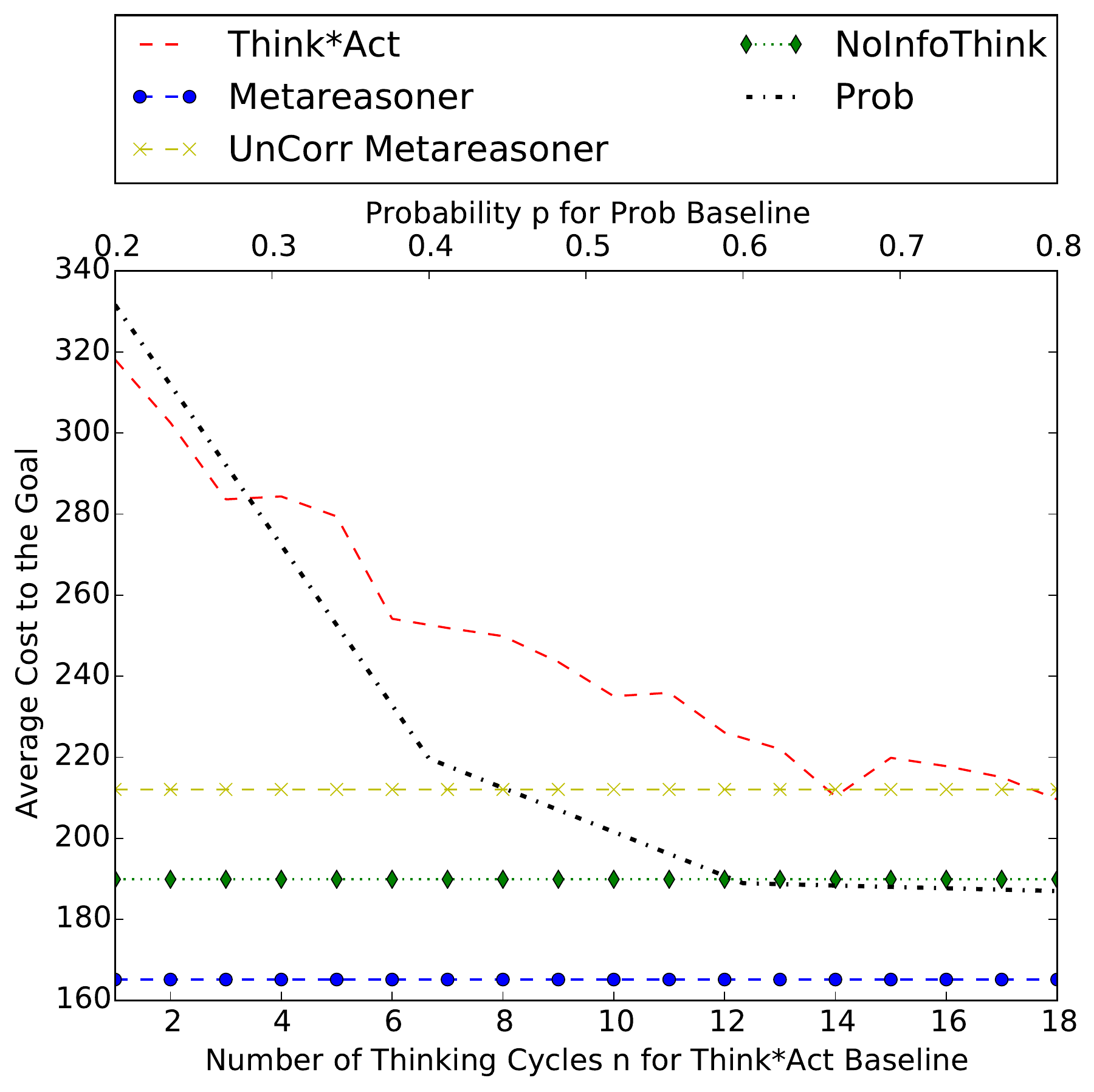}
\caption{Comparison of algorithms in \emph{Stochastic}, with the cost of acting = 15 and cost of thinking = 1}
\label{winds-dCOA-15}
\end{figure}
\comment{
\begin{figure}
\centering
\includegraphics[scale = 0.45]{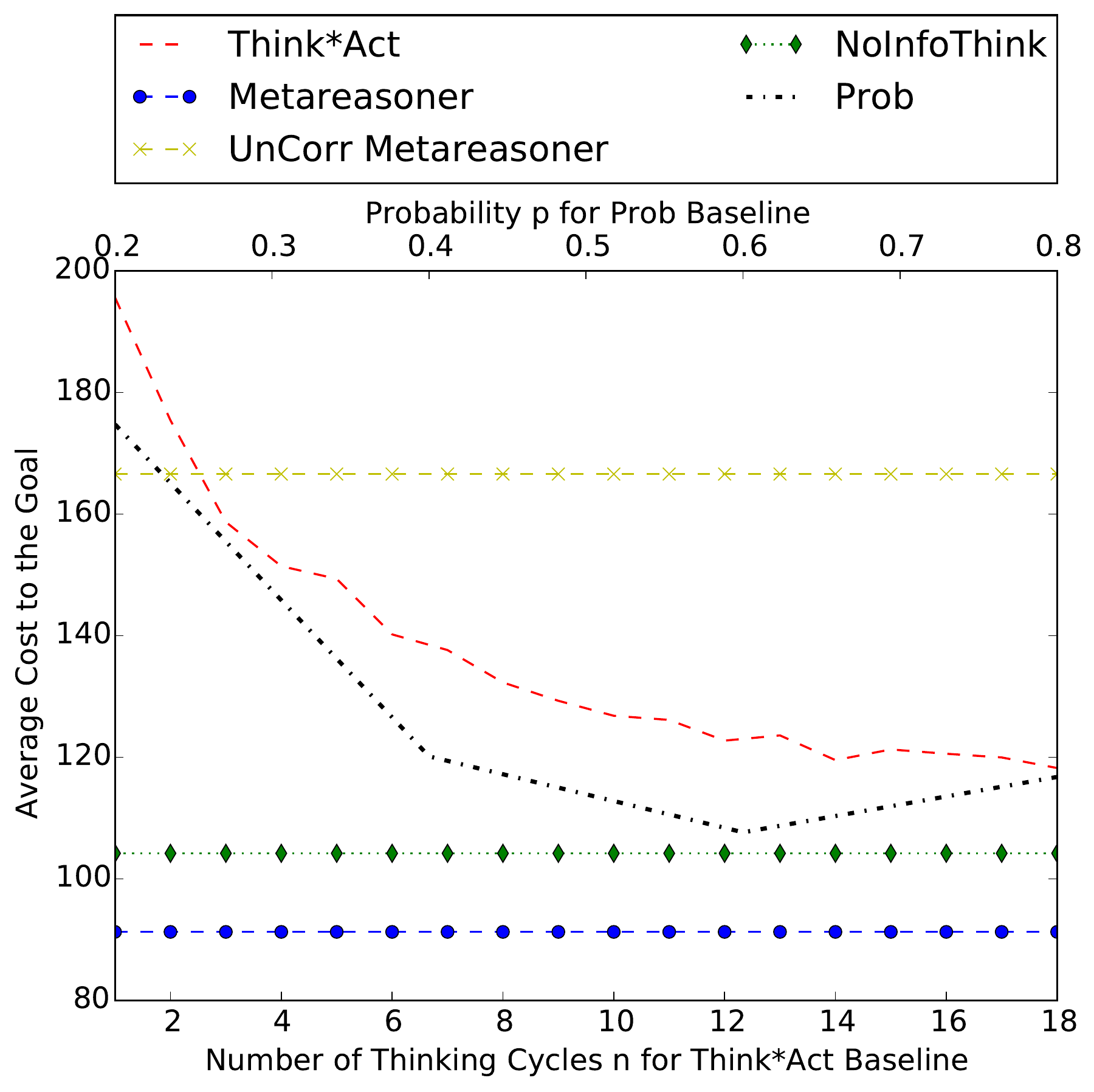}
\caption{When its performance is averaged over all settings of the cost of acting in \emph{Stochastic} (and the cost of thinking is fixed at 1), our algorithm yields a lower cost to the goal than any of the baselines with any parameter settings}
\label{winds-dCOA-average}
\end{figure}
\begin{figure}
\centering
\includegraphics[scale = 0.45]{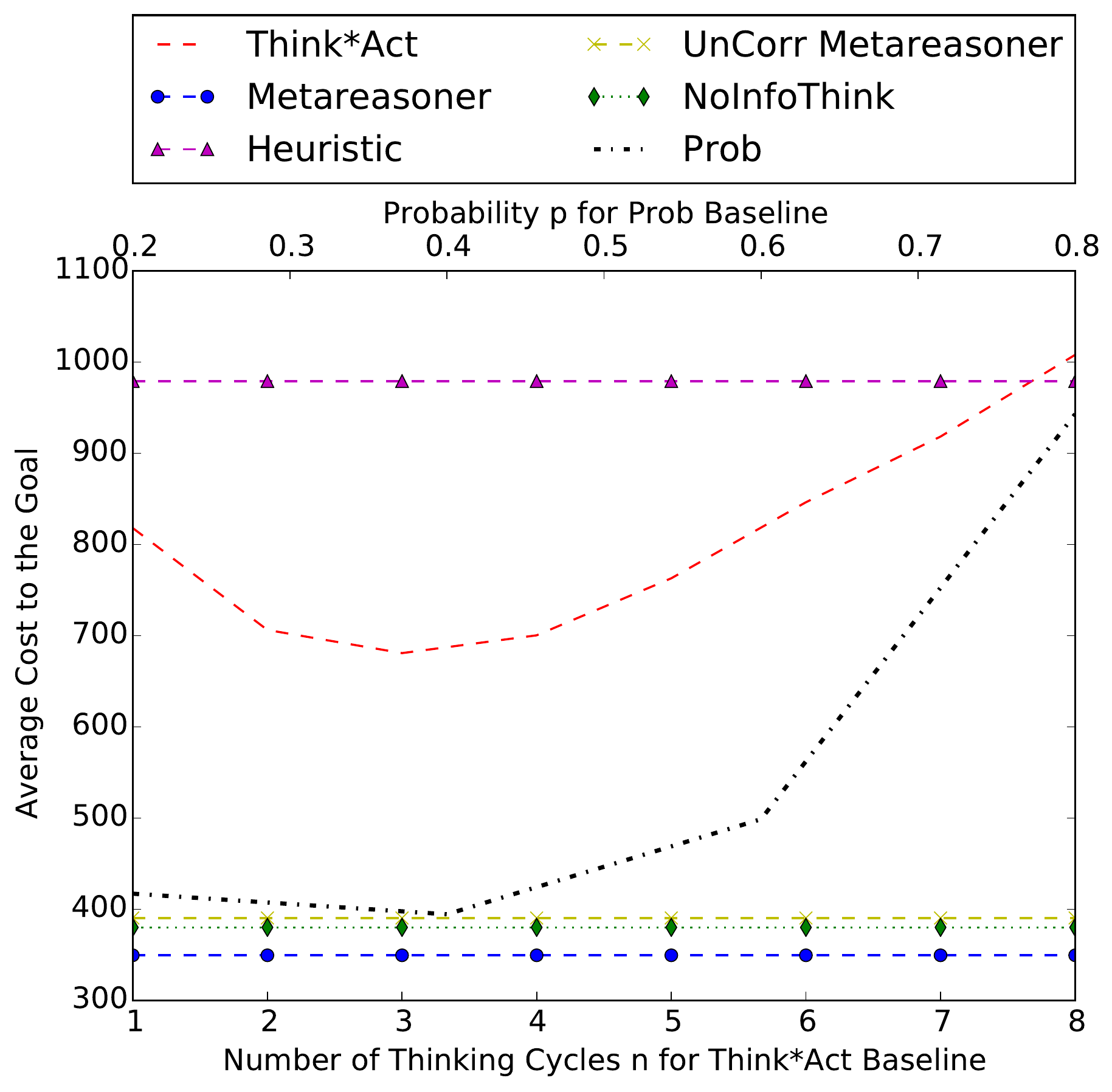}
\caption{Comparison of algorithms with baselines at the \emph{Traps} domain.}
\label{lavapit}
\end{figure}
\begin{figure}
\centering
\includegraphics[scale = 0.45]{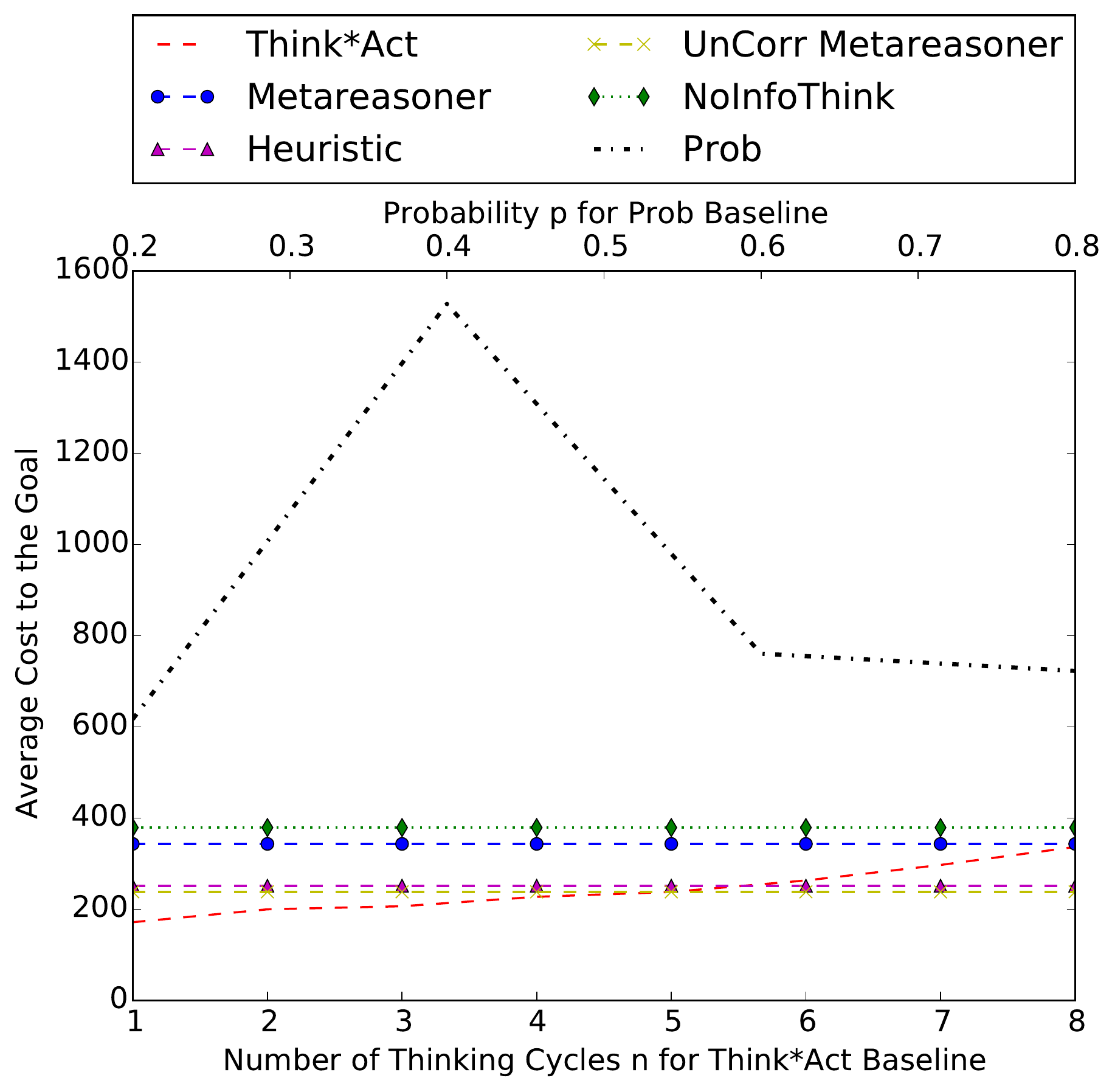}
\caption{Comparison of algorithms in \emph{DynamicNOP-1}.}
\label{fireants1}
\end{figure}
\begin{figure}
\centering
\includegraphics[scale = 0.45]{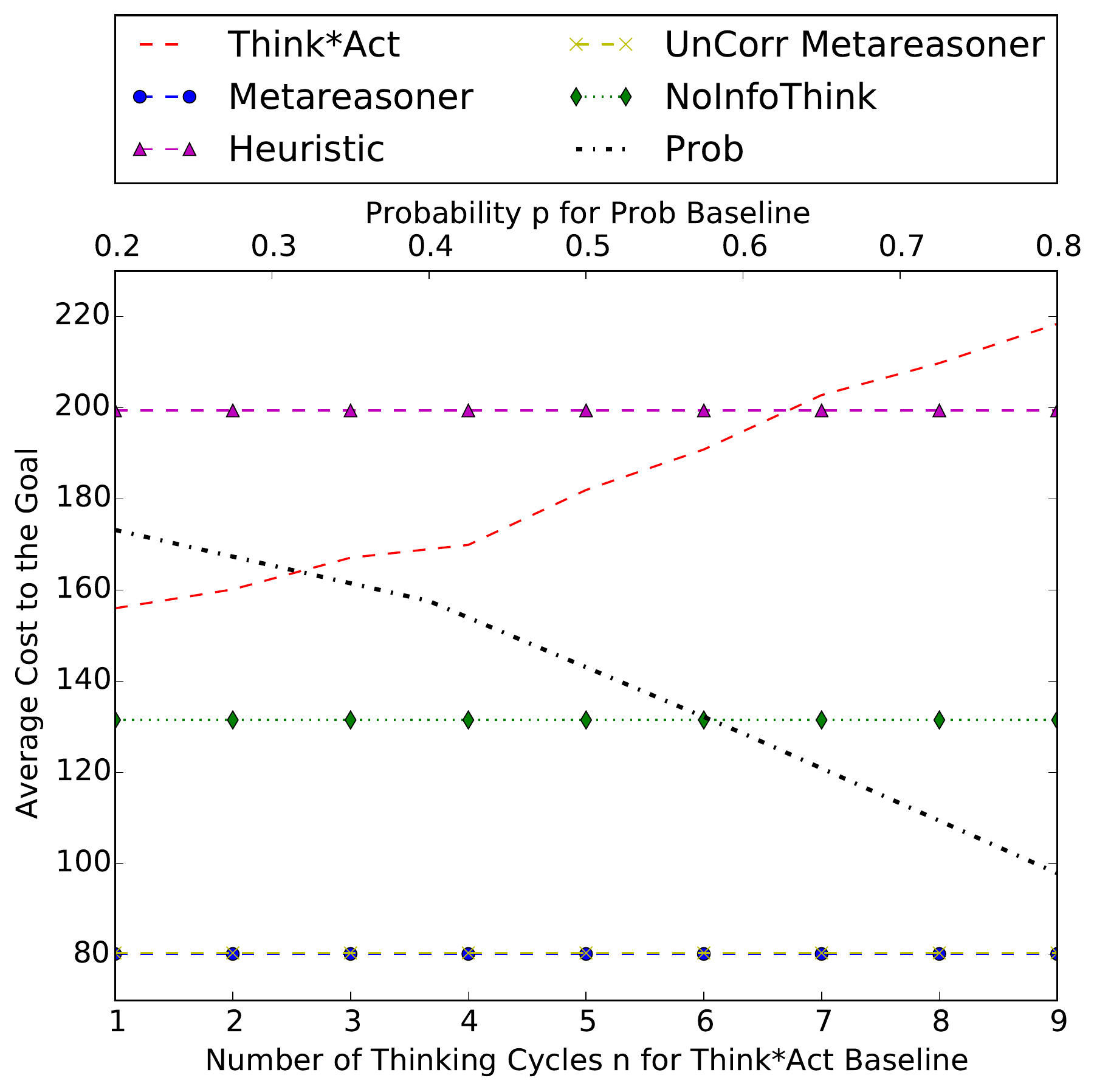}
\caption{Comparison of algorithms in \emph{DynamicNOP-2}.}
\label{fireants2}
\end{figure}}

\end{document}